%% file: corruption_bai_v26_for_arxiv_upload.tex
\documentclass{article}  
  
\usepackage{microtype}
\usepackage{graphicx}
\usepackage{subfigure}
\usepackage{booktabs} 

\usepackage{hyperref}

\usepackage[accepted]{icml2021_rev}

\input{preamble_icml_arxiv}
\usepackage[skip=0.1in]{caption}
\usepackage{enumitem}

\usepackage{hhline}
\makeatletter
\newcommand{\thickhline}{%
    \noalign {\ifnum 0=`}\fi \hrule height 1.5pt
    \futurelet \reserved@a \@xhline
}
\newcolumntype{"}{@{\hskip\tabcolsep\vrule width 1pt\hskip\tabcolsep}}
\makeatother

\newcommand{\iout}{i_{ \mathrm{out} }}
\newcommand{\gapbd}{ \epsilon_C }
\newcommand{\errpr}{ \delta}
\newcommand{\tildeW}{\tilde{W}}

\allowdisplaybreaks

\icmltitlerunning{Probabilistic Sequential Shrinking: A Best Arm Identification Algorithm for Stochastic Bandits with Corruptions}

\begin{document}

\twocolumn[
\icmltitle{Probabilistic Sequential Shrinking: A Best Arm Identification Algorithm for Stochastic Bandits with Corruptions}



\icmlsetsymbol{equal}{*}

\begin{icmlauthorlist}
\icmlauthor{Zixin Zhong}{math}     
\icmlauthor{Wang Chi Cheung}{ise,iora}  
\icmlauthor{Vincent Y.~F.~Tan}{math,iora,ece}       
\end{icmlauthorlist}

\icmlaffiliation{math}{Department of Mathematics, National University of Singapore, Singapore}
\icmlaffiliation{ise}{Department of Industrial Systems and Management, National University of Singapore, Singapore}
\icmlaffiliation{iora}{Institute of Operations Research and Analytics, National University of Singapore, Singapore}
\icmlaffiliation{ece}{Department of Electrical and Computer Engineering, National University of Singapore, Singapore}
\icmlcorrespondingauthor{Zixin Zhong}{zixin.zhong@u.nus.edu}
\icmlcorrespondingauthor{Wang Chi Cheung}{isecwc@nus.edu.sg}
\icmlcorrespondingauthor{Vincent Y.~F.~Tan}{vtan@nus.edu.sg}

\icmlkeywords{Machine Learning, ICML}

\vskip 0.3in
]



\printAffiliationsAndNotice{}  

 \begin{abstract}
We consider a best arm identification (BAI) problem for stochastic bandits with adversarial corruptions in the fixed-budget setting of $T$ steps. 
We design a novel randomized algorithm, {\sc Probabilistic Sequential Shrinking$(u)$} ({\sc PSS}$(u)$), which is agnostic to the amount of corruptions. 
When the amount of corruptions per step (CPS) is below a threshold, {\sc PSS}$(u)$ identifies the best arm or item with probability tending to $1$ as $T\rightarrow\infty$. Otherwise, the optimality gap of the identified item degrades gracefully with the CPS. We argue that such a bifurcation is necessary. 
In {\sc PSS}$(u)$, the parameter $u$ serves to balance between the optimality gap and success probability. The injection of randomization is shown to be essential to mitigate the impact of corruptions. 
%
To demonstrate this, we design two attack strategies that are applicable to any   algorithm. We apply one of them to a deterministic analogue of {\sc PSS}$(u)$ known as   {\sc Successive Halving} ({\sc SH})   by \citet{karnin2013almost}. The attack strategy results in a high failure probability for {\sc SH}, but {\sc PSS}$(u)$ remains  {robust}. In the absence of corruptions, {\sc PSS}$(2)$'s performance guarantee matches SH's. 
%
%
%
%
We  show that when the CPS is sufficiently large, no algorithm can achieve a BAI probability tending to $1$  as $T\rightarrow \infty$. Numerical experiments corroborate our theoretical findings.

\end{abstract}
\vspace{-.2in}

\section{Introduction}
\vspace{-.05in}

Consider a drug company $\mathrm{D}_1$ that wants to design a vaccine for a certain illness, say COVID-19. It has a certain number of options, say $L=10$, to design a near-optimal vaccine. Because $\mathrm{D}_1$ has a limited budget, 
it can only test vaccines for a fixed number of times, say $T=1000$.
Using the limited number of tests,  it wants to find the option that will lead to the ``best''   outcome, e.g., the shortest average recovery time of certain model organisms. 
%
%
However, vaccine trials usually assume that every test subject satisfies a certain set of criteria, such as having no prior related illnesses. 
If a subject who violated the criteria is tested, the observed recovery time would be
{\em corrupted}.
%
We assume the total corruption budget is bounded as  a function of the number of tests.
How can $\mathrm{D}_1$ find a near-optimal drug design in the presence of the corruptions and uncertainty of the efficacy of the drugs? We will show that the utilization of a suitably {\em randomized} algorithm is {assumed to be key.}
 

 To solve  $\mathrm{D}_1$'s problem,  we study the {\em Best Arm Identification} (BAI) problem for stochastic bandits with adversarial corruptions. We note that the effect and mitigation of corruptions were studied for the {\em Regret Minimization} problem by \citet{lykouris2018stochastic} and others. 
  While most existing works study the BAI problem for   stochastic bandits {\em without} corruptions~\citep{auer2002finite,audibert2010best, carpentier2016tight}, 
 \citet{altschuler2019best} considers a variation of the classical BAI problem and aims to identify an item with high {\em median} reward, while \citet{shen2019universal} assumes that the amount of corruption  per step~(CPS) {\em diminishes} as time progresses.
Therefore, these studies are not directly applicable to $\mathrm{D}_1$ as we are interested in obtaining a near-optimal item in terms of the mean and we assume that the CPS does not diminish with time.
Our setting dovetails neatly with company $\mathrm{D}_1$'s problem  
and $\mathrm{D}_1$ can utilize our algorithm to sequentially and adaptively select different design options  to test the vaccines and to eventually choose  a near-optimal design that results in a short recovery time even in the presence of adversarial corruptions.

Beyond any specific applications,  we believe that this problem is of fundamental theoretical importance in the broad context of BAI in multi-armed bandits (MAB) and adversarial machine learning. In particular,  \citet{gupta2019better} advanced the theory of regret minimization in MAB; this work complements Gupta's work in the BAI setting. 

\noindent {\bf Main Contributions.} 
In stochastic bandits with adversarial corruptions, there are $L$ items with different  {rewards} distributions. At each time step, a random reward is generated from each item's distribution; this reward is observed and arbitrarily corrupted by the adversary. 
The learning agent selects an item based on corrupted observations in previous steps, and only observes the pulled items' corrupted rewards.
Given $T\in\mathbb{N}$, the agent aims to identify a near-optimal item
with high probability over $T$ time steps.
Our first main contribution is the {\sc Probabilistic Sequential Shrinking$(u)$} ({\sc PSS}$(u)$) algorithm. 
{\sc PSS}$(u)$ is agnostic to the amount of adversarial corruption. The parameter $u$ can be adjusted to trade-off between the optimality gap of the identified item and the success probability.





The key challenge 
lies in mitigating the impact of corruptions. For this purpose, upon observing pulled items' corrupted rewards in previous time steps, {\sc PSS}$(u)$ pulls subsequent items probabilistically.
By comparing {\sc PSS$(u)$} to 
{the state-of-the art for BAI with fixed horizon, namely}
{\sc Successive Halving} ({\sc SH}) by \citet{karnin2013almost},
we argue that {randomization} is beneficial, and indeed necessary, for 
BAI under adversarial corruption. 
{On one hand, {\sc PSS}$(2)$'s failure probability in BAI (at least in the exponent) matches that of {\sc SH} when there is no corruption. On the other hand, the largest possible CPS under which {\sc PSS}$(2)$ succeeds in BAI with probability $1 - \exp(-\Theta(T))$ is  a factor of $L$ larger than that for SH.} En route, we identify a term in the exponent of the failure probability of {\sc PSS}$(u)$ that generalizes the ubiquitous $H_2$ term for BAI under the fixed-budget setting. {Finally, when CPS is so large that BAI is impossible, the sub-optimality gap of the identified item degrades gracefully with the CPS.} In complement, we provide lower bound examples to show that BAI is impossible when CPS is sufficiently large.  
The examples involve judiciously chosen attack strategies, which
corroborate the tightness of our performance guarantee for {\sc PSS}$(u)$. {Numerical experiments on various settings further corroborate our theoretical findings.}

%

\textbf{Novelty.} (i) We identify randomization as a key tool in mitigate corruption in BAI, and 
identify an achievable sub-optimality gap for PSS$(u)$. 
(ii) The analysis of PSS$(u)$ shows how our designed randomization ``confuses the adversary'', which results in the improvement over SH, and yields (suboptimality gap and failure probability exponent) results that are almost {\emph{tight}} with respect to the lower bounds. 
(iii) The design of the attack strategies, which involves a randomized adversary, and their analysis are novel. 


 
\noindent {\bf Literature review.}
The BAI problem has been studied extensively for  both stochastic bandits~\citep{audibert2010best} and   bandits with adversarial corruptions~\citep{shen2019universal}.
There are two complementary settings for BAI:
(i) Given $T \in\mathbb{N}$, the agent aims to maximize the probability of finding a near-optimal item in at most $T$ steps;  
(ii) Given $\delta  > 0$, the agent aims to find a near-optimal item with the probability of at least $1-\delta$ in the smallest number of steps.
These settings are respectively known as the {\em fixed-budget} and {\em fixed-confidence} settings.
 Another line of studies aims to prevent the agent from achieving the above desiderata and thus to design strategies to {\em attack} the rewards efficiently~\citep{jun2018adversarial,liu2020action}.
We now review some 
works. 

First, we review  related work in  stochastic bandits.
Both the fixed-budget setting~\citep{audibert2010best,karnin2013almost, 
jun2016top}
and the fixed-confidence setting~\citep{audibert2010best,NIPS2014_5433,rejwan2020top,zhong2020best} have been extensively studied.
However, as previously motivated,  
we need to be cognizant that the agent may encounter corrupted rewards and thus must design appropriate strategies to nullify or minimize the effects of these corruptions.  

 Regret minimization on stochastic bandits with corruptions was first studied by \citet{lykouris2018stochastic}, and has attracted extensive interest recently~\citep{zimmert2018tsallisinf,li2019stochastic,gupta2019better,lykouris2020corruption,liu2020action,krishnamurthy2020binsearch,Bogunovic0S20}. 
Pertaining to the BAI problem in the presence of corruptions, 
 \citet{altschuler2019best} studies a variation of the classical fixed-confidence setting and aims to find an item with a high median reward.
In contrast, \citet{shen2019universal} proposes an algorithm under the fixed-budget setting, whose theoretical guarantee requires a number of  stringent conditions. 
In particular, \citet{shen2019universal} assumes that CPS
diminishes as time progresses.
However, it may be hard to verify in practice  whether these conditions  are satisfied. 
In spite of the many existing works, the classical BAI problem 
has not been analyzed when the rewards suffer from general corruptions.
Our work fills in this gap in the literature by proposing and analyzing the {\sc PSS}$(u)$ algorithm under the fixed-budget setting. The randomized design of our algorithm is crucial in mitigating the impact of corruptions.

{Another concern 
 is how an adversary can corrupt the rewards to prevent the agent from obtaining sufficient information from the corrupted observations.
Many studies aim at attacking certain algorithms, such as UCB, $\epsilon$-greedy or Thompson sampling, using an adaptive strategy~\citep{jun2018adversarial,zuo2020near}. \citet{liu2019data} 
design offline strategies to attack a particular algorithm and also an adaptive strategy against any algorithm.
All these strategies aim to corrupt the rewards such that the agent can only receive a small cumulative reward in expectation. The  design and analysis
 of attack strategies pertaining to the BAI problem have been unexplored.
Our analysis fills in this gap by proposing two offline strategies for  Bernoulli instances and proving that when the total corruption budget is sufficiently large (i.e., of the order\footnote{A (non-negative) function  $f(T) = O(T)$ if there exists a constant $0<c<\infty$ (dependent on $w$ but not $T$) such that $f(T)\le cT$ for sufficiently large $T$. Similarly $f(T)=\Omega(T) $ if there exists $c>0$ such that $f(T)\ge cT$ for sufficiently large $T$. Finally, $f(T)=\Theta(T)$ if $f(T)=O(T)$ and $f(T)=\Omega(T)$.} 
 $\Omega(T)$), {\em any} algorithm will fail to identify any near-optimal item with 
 constant  probability.}

\section{Problem Setup}
\label{sec:prob_setup}

For brevity,  for any $n\in\mathbb{N}$, we denote the set $\{ 1 , \ldots, n \}$ as $[n]$.
    Let there be $L\in\mathbb{N}$ ground items, contained in $[L] $. Each item $i\in [L]$ is associated with a reward distribution $\nu(i)$ supported in $[0,1]$ with {mean $ w(i)$}.
The distributions  $\{\nu(i)\}_{i\in [L]}$ and means $\{w(i)\}_{i\in [L]}$ are not known to the agent.  
Over time, the agent is required to 
{identify the best or close-to-best ground item}
by 
adaptively pulling items. The agent aims to identify an optimal item, which is an item of the highest mean reward, after a fixed time budget of $T\in \bbN$ time steps, whenever possible in the presence of corruptions. 
More precisely, at each time step $t\in[T]$, 
%
\begin{enumerate}[ itemsep = -1pt,   topsep = -1pt, leftmargin =  18pt ]
    \item[(i)] A stochastic reward $W_t(i) \! \in \! [0, 1] $ is drawn for each item $i$ from $\nu(i)$. 
    \item[(ii)] The adversary observes $\{W_t(i)\}_{i\in [L]} $, and corrupts each $W_t(i)$ by an additive amount $c_t(i)\in [-1,1]$, leading to the corrupted reward $\tilde{W}_t(i) = W_t(i) + c_t(i) \in [0, 1]$ for each $i\in [L]$. 
    \item[(iii)] The agent pulls an item $ i_t \in [L]$ and observes the corrupted reward $\tilde{W}_t( i_t )$.
\end{enumerate}
For each $i \in [L]$, the random variables in $\{W_t(i)\}^T_{t=1} $ are i.i.d.
When determining $\{c_t(i)\}_{i\in [L]}$ at time step $t$, the adversary cannot observe the item $i_t$ going to be pulled, 
but he can utilize the current observations consisting of $\{W_q(1), \ldots, W_q(L) \}^t_{q=1}$, $\{c_q(1), \ldots, c_q(L) \}^{t-1}_{q=1}$, and  $\{i_q\}^{t-1}_{q=1}$.  
We assume that the total amount of adversarial corruptions during the horizon is bounded:
\begin{align*}
    \sum_{t=1}^T \max_{ i \in [L] } | c_t(i) |  \le C.
\end{align*}
The \emph{corruption budget} $C$ is not known to the agent.

 We focus on instances with a unique item of the highest mean reward, and assume that $w(1) > w(2) \ge \ldots \ge w(L)$, so that item $1$ is the unique optimal item. 
 To be clear, the items can, in general, be arranged in any order; the ordering that $w(i)\ge w(j)$ for $i< j$ is just to ease our discussion.
We denote $\Delta_{1,i} = w(1) -w(i)$ as the {\em optimality gap} of item $ i $.  
 An item $i$ is {\em $\epsilon$-optimal} ($\epsilon\ge 0$) if $\Delta_{1,i} \le \epsilon$.

The agent uses an {\em online algorithm} $\pi$ to decide the item $i_t^\pi$ to pull at each time step $t$, and the item $\iout^{ \pi, T }$ to output as the identified item eventually.
More formally, an online algorithm consists of  a tuple $\pi:= ( (\pi_t)_{t=1}^T,  \phi^{\pi,T} )$, where
\begin{itemize}[ itemsep = 0pt,   topsep = -1pt, leftmargin =  10pt ]
    \item the {\em sampling rule} $\pi_t$ determines, based on the observation history, the item $i_t^\pi $ to pull at time step $t$. That is, the random variable $i_t^\pi$ is $\calF_{t-1}$-measurable, where $\calF_t :=  \sigma ( i_1^\pi,  \tilde{W}_1( i_1^\pi ), \ldots,  i_{t}^\pi,  \tilde{W}_t( i_t^\pi ) ) $;
    \item the recommendation rule $\phi^{ \pi , T} $ chooses an item $\iout^{\pi,T}$, that is, by definition, $\calF_T$-measurable. 
\end{itemize}
We denote the probability law of the process $\{ \tilde{ \bW }_t  = (\tilde{ W}_t (1), \ldots, \tilde{W}_t (L) )  \}_{t=1}^T$ by $\bbP$. This probability law depends on the agent's online algorithm $\pi$, which influences the adversarial corruptions. 

 For fixed $\epsilon_C, \delta\in (0, 1)$, an algorithm $\pi$ is said to be {\em $(\gapbd, \errpr)$-PAC (probably approximately correct)} if 
$$ \bbP \big[\Delta_{ 1, \iout^{\pi,T} } > \gapbd \big] \le \errpr.$$ 
Our overarching goal is to design an $(\gapbd, \errpr)$-PAC algorithm $\pi$ such that both $ \gapbd$
and $\errpr$ are small. In particular, when $\epsilon_C < \Delta_{1, 2}$, an $(\epsilon_C , \delta)$-PAC algorithm $\pi$ identifies the optimal item with probability at least $1 -\delta$. For BAI with no corruption, existing works \citep{audibert2010best,karnin2013almost} provide $(0, \exp [-\Theta(T)] )$-PAC algorithms.
In the presence of corruptions, unfortunately it is impossible to achieve a $(0, \exp [-\Theta(T)] )$-PAC performance guarantee, as we discuss in the forthcoming Section \ref{sec:main_result_lb}. We investigate the trade-off between $\epsilon_C$ and $\delta$, and focus on constructing $(\epsilon_C, \exp [-\Theta(T)] )$-PAC algorithms with $\epsilon_C$ as small as possible. 
We abbreviate 
$i_t^\pi$ as $i_t$ and $\iout^{\pi, T } $ as $\iout$ when there is no ambiguity.



 Finally, in anticipation of our main results, we remark that given a failure probability $\delta$,  the smallest possible $\gapbd$ is, in general, a function of the \emph{corruption per step} (CPS) $C/T$ 
and {\em possibly} the total number of items $L$.  
   

\section{Algorithm}
\label{sec:alg}

 Our algorithm {\sc Probabilistic Sequential Shrinking} $\!\!(u)$ (PSS$(u)$) is presented in Algorithm~\ref{alg:probaSS}. 
 The algorithm involves
randomization in order to mitigate the impact  of adversarial corruptions.

\begin{algorithm}[ht] 
	\caption{{\sc Probabilistic Sequential Shrinking}} \label{alg:probaSS}
		\begin{algorithmic}[1]
			\STATE {\bfseries Input:}  time budget $T$, size of ground set of items $L$, parameter $u\in (1, L]$.
			\STATE Set $M =  {\lceil} \log_u L  {\rceil}$, $N  = {\lfloor} { T }/{M} {\rfloor}$, $T_0=0$, $A_0 = [L]$.
			\FOR{phase $m = 1, 2, \ldots,M$} 
                \STATE Set $T_m  \!=  \! T_{m-1} +   N, q_m \!  ={1} / { | A_{m-1}  | }, n_m \! = q_m  N $.
			    \FOR{$t = T_{m-1}+1,  \ldots, T_{m}$} 
			        \STATE Choose item $i \in A_{m-1}$ with probability $q_m$, pull it  and observe corrupted reward $\tilde{W}_t(i)$. \label{line:probaSS_prob_pull}
			    \ENDFOR			
			    \STATE For all $i\in A_{m-1}$, set   \newline
			        \resizebox{.43\textwidth}{!}{
			        $ \displaystyle
			            S_m(i)  =\! \!\sum_{ t =T_{m-1} + 1 }^{ T_m } \!\! \tilde{W}_t(i_t) \cdot \mathbb{I} \{ i_t = i \},
			            \ 
			            \hat{w}_m(i)  = \frac{ S_m(i) }{ n_m }.
			        $}			        
			    \STATE Let $A_{m}$ contain the $ \lceil  L/ u^{m}  \rceil $ items with the highest empirical  means $\hat{w}_m(i)$'s in $ A_{m-1}$.
			\ENDFOR
			\STATE Output the single item $\iout \in A_M$. 
		\end{algorithmic}
	\end{algorithm}
 The agent partitions the whole horizon into $\lceil \log_u L \rceil$ phases of equal length. 
 During each phase, PSS$(u)$ classifies an item as {\em active} or {\em inactive} based on the empirical averages of the corrupted rewards. 
Initially, all ground items are active and belong to the {\em active set} $A_0$.
 Over phases, the active sets $A_m$ shrink, and an item may be eliminated from $A_m$ and consequently it may become {\em inactive}.	
	

During phase $m$:
\begin{itemize}[ itemsep = -1pt,   topsep = -1pt]
    \item[(i)] at each time step, the agent chooses an active item uniformly at random and pulls it;
     \item [(ii)] at the end, the agent finds $\hat{w}_m(i)$, the corrupted empirical mean during phase $m$ for each active item $i$;
    \item [(iii)] the agent utilizes the $\hat{w}_m(i)$'s of active items $i \in A_{m-1}$ to shrink the active set. 
\end{itemize}
By the end of  the last phase $M$, we show that $|A_M|=1$ (see Lemma~\ref{lemma:ub_probaSS_single_output} in Section~\ref{sec:pf_sketch_ub_probaSS}), and the agent outputs the single active item. 

The effectiveness of Algorithm~\ref{alg:probaSS} is manifested in four different aspects: 
(i) the agent only utilizes information from the {\em current} phase to shrink the active set, which ensures that any corruption has a limited impact on her decision;
 (ii) the injection of randomization by the agent to decide on which item to pull nullifies the ability of the adversary from corrupting rewards of {\em specific} items; 
(iii) the agent can handle the adversarial attacks even though she does not know the total corruption budget $C$;
(iv) the agent can choose any $u\in (1,L]$ to trade off between
$\gapbd$ and $\errpr$ in its $(\gapbd, \errpr)$-PAC performance guarantee.
A smaller $\gapbd$ leads to a higher failure probability $\errpr$. 
We would like to emphasize that though the agent can choose any $u\in (1,L]$, this parameter  is a fixed constant and cannot vary with the horizon $T$ after PSS($u$) is initialized.

When $u=L$, PSS$(L)$ regards the  horizon $T$ as a single phase. 
Each item is pulled with probability $1/L$ at each step, and is expected to be pulled for $T/L$ times in $T$   steps. We can regard PSS$(L)$ as a randomized version of the na\"ive {\sc Uniform Pull}~(UP) algorithm, which pulls each item for $ \lfloor T/L \rfloor $ times according to a deterministic schedule.


When $u = 2$, $\text{PSS}(2)$ is a randomized analogue to the {\sc Sequential Halving}~(SH) algorithm proposed in \citet{karnin2013almost}.
Both PSS$(2)$ and SH divide the whole horizon into $\lceil \log_2 L \rceil$ phases and halve the active set during each phase, i.e., $A_m = \lceil L/  2^m   \rceil$.  
However, the differences between them are as follows:
\begin{itemize}[ itemsep = .5pt,   topsep = -1pt , leftmargin = 11pt]
    \item at each time step of phase $m$, PSS$(2)$ chooses item $ i \in A_{m-1}$ with probability $1/|A_{m-1}|$ and pulls it (Line~\ref{line:probaSS_prob_pull} of Algorithm~\ref{alg:probaSS});
    \item during phase $m$, SH pulls each item in $A_{m-1}$ for {\em exactly} $ \lfloor T/( \lceil \log_2 L \rceil \cdot |A_{m-1}| ) \rfloor $ times according to a deterministic schedule.
\end{itemize}
Therefore, though PSS$(2)$ and SH pull each active item for about an equal number of times in expectation, PSS$(2)$ involves more randomness in the pulls.
%

\section{Main Results}

\subsection{Upper Bound}


\begin{restatable}{theorem}{thmUbprobaSS}

\label{thm:ub_probaSS}
For any $u \in (1,L]$, 
the {\sc Probabilistic Sequential Shrinking$(u)$} algorithm, as presented in Algorithm~\ref{alg:probaSS}, outputs an item $\iout$ satisfying  
\begin{align}
    & 
    \bbP\bigg[  \Delta_{ 1 , \iout } > \frac{ 8C \lceil \log_u L \rceil }{   T }    \bigg]
    \nonumber\\
    &\! \le \! 4 \lceil \log_u L \rceil (L \!-\! 1) \exp \bigg[ \!  - \!  \frac{   1  }{ 192 \tilH_2 (w,L,u)   }  \! \cdot \! \Big\lfloor \frac{ T}{ \lceil \log_u L \rceil } \Big\rfloor   \bigg]   
    \nonumber\\
    &\! =\! O \bigg( \!
       L (\log_u L)  \exp \bigg[ \!-\! \frac{ T}{  192 \tilH_2(w,L,u) \! \log_u L  } \!  \bigg]
      \bigg),\!\!\label{eq:delta_bound}
\end{align}
where
\begin{align}
    \tilH_2 (w,L,u)=   
                  \max_{ 
                      i \ne 1
                  }
                  \frac{ \min\{ u  \cdot i,~ L \} }{ \Delta_{1,i}^2  }.
    \label{eq:ub_probaSS_def_HtwoPrime}
\end{align} 

\end{restatable}

Theorem \ref{thm:ub_probaSS} shows that PSS$(u)$ is $(\epsilon_C, \delta)$-PAC for any $u\in (1, L]$, where $$\epsilon_C = O\left(\frac{C \log L}{T}\right) \quad
\text{ and }\quad\delta = \exp[-\Theta(T)].$$ We remark that only $\gapbd$, but not $\errpr$, depends on $C$. The dependence of $\gapbd$ on the CPS  $C/T$ is, in general, unavoidable in view of our lower bounds (see Section~\ref{sec:main_result_lb}).


The upper bound on the failure probability $\errpr$ involves the parameter $\tilde{H}_2(w,L,u)$, which quantifies the difficulty of identifying the best item in the instance. The parameter $\tilde{H}_2(w,L,u)$ generalizes its analogue 
\begin{align*}
    H_2(w) = \max_{  
                      i \ne 1
                  }
                  \frac{i  }{ \Delta_{1,i}^2  } 
\end{align*}
proposed by \citet{audibert2010best},
in the sense that
$$ \lim_{ u \rightarrow 1^+ } \tilH_2(w,L,u) = H_2(w),\quad\forall\, w \in [0,1]^L .$$
We propose to consider the more general version $\tilH_2(w,L,u)$ in order to analyze the randomized versions of SH and UP under one unified framework.



%

\noindent {\bf Function of parameter $u$.}
Theorem~\ref{thm:ub_probaSS} implies that when $u$ increases, the upper bound $\gapbd$ on $\Delta_{1, \iout}$ decreases.
However, the quantity $\tilH_2 (w,L,u) $ increases,
which leads to a larger upper bound on the failure probability.
Specifically, 
\begin{align*}
    \tilH_2 (w, L , u_2 ) \ge \frac{ u_2 }{ u_1 }  \tilH_2 (w, L , u_1 ), ~~ \forall \, 1 < u_1 \le u_2 \le L.
\end{align*}
Meanwhile, as presented in Algorithm~\ref{alg:probaSS}, PSS$(u)$ with a larger $u$ separates the whole horizon into fewer phases and shrinks the active set faster.
(i) The fewer number of phases leads to a longer duration of each phase, which is beneficial for bounding the impact of corruptions (see Lemma~\ref{lemma:ub_probaSS_conc_apply}).
(ii) Besides, the faster the active sets shrink, the larger $\tilH_2  (w,L,u) $ is.
See Section~\ref{pf:ub_probaSS_thm_final_step} for details.

Altogether, Theorem~\ref{alg:probaSS} 
provides a bound on learning an $\gapbd$-optimal item and
implies that PSS$(u)$ allows the agent to trade off between the bound on $\Delta_{1, \iout}$ and the failure probability by adjusting $u$. When the CPS is so low that 
\begin{align}
  \frac{C}{ T } <  \frac{   \Delta_{1,2 }  }{ 8 \lceil \log_u L \rceil },  
    \label{eq:probaSS_C_OK}
\end{align}
 Theorem \ref{thm:ub_probaSS} implies that $\text{PSS}(u)$ identify the optimal item with probability at least $1 - \delta$, where $\delta =\exp(-\Theta(T))$ is as shown in \eqref{eq:delta_bound}. When the CPS is so large such that
\begin{align}
    \frac{C}{ T } \ge  \frac{   \Delta_{1,L }  }{ 8 \lceil \log_u L \rceil },  
    \label{eq:probaSS_C_too_large}
\end{align}
Theorem~\ref{thm:ub_probaSS} is vacuous, since all the items are $\Delta_{1,L}$-optimal.  
In the extreme case in which 
 $$ \frac{C}{ T } \ge  \sup_{ u\in (1,L]  } \frac{      \Delta_{1,L }  }{ 8 \lceil \log_u L \rceil }    
     =  \frac{   \Delta_{1,L} }{  8 } ,
     $$ 
     Theorem \ref{thm:ub_probaSS} is vacuous for all $u\in (1, L]$. 
Indeed, we show in Section~\ref{sec:main_result_lb} 
that this bifurcation on the learnability holds true not only to our algorithms. No algorithm can achieve BAI when $C/T$ is above a certain threshold. In passing, our characterization of the threshold is tight up to log factors. 


\noindent {\bf BAI on stochastic setting without corruptions.}
In the setting without adversarial corruptions, i.e., $C=0$, Theorem~\ref{alg:probaSS} upper bounds the probability that PSS$(u)$ outputs $\iout$ with $\Delta_{ 1,\iout } >0$. We compare Theorem \ref{thm:ub_probaSS} on $\text{PSS}(2)$ with the performance guarantee of SH by \citet{karnin2013almost}:  
\begin{align*}
    \bbP [  \Delta_{1,\iout} > 0  ]
    = O 
    \bigg(
         (\log_2 L )  \exp \bigg[   - \frac{ T   }{  8 H_2(w )  \log_2 L  }   \bigg]
    \bigg) .
\end{align*}
Disregarding constants, the bound on $\bbP [   \Delta_{1,\iout} > 0   ]$ of PSS$(2)$ is worse than that 
of SH by a factor of $L$, which is a multiplicative factor we incur due to the impact of corruptions.
Apart from that, our bound involves $\tilH_2(w,L,2)$ while \citet{karnin2013almost} involves $H_2(w)$, and notice that
$$\tilH_2 (w,L,2) \le 2 H_2(w) .$$
As a result, our exponent matches that by \citet{karnin2013almost} up to an absolute constant (which is $48$).

 Next, we compare Theorem \ref{thm:ub_probaSS} on $\text{PSS}(L)$ with the performance guarantee of UP, which is folklore. We use the following in Section 33.3 of \citet{lattimore2020bandit}:
\begin{align}
\!\!\!\! \Pr[\Delta_{1, \text{out}} > 0]  &\leq \sum^L_{i=2}\exp\left[ - \frac{ \lfloor T / L \rfloor \cdot  \Delta^2_{1, i}}{4}\right] \nonumber\\*
& \leq  (L-1) \exp\left[ - \frac{  \lfloor T / L \rfloor \cdot \Delta^2_{1, 2}}{4}\right], \label{eq:UP_loose_bound}
\end{align}
where \eqref{eq:UP_loose_bound} is tight when $\Delta_{1, 2} = \Delta_{1, i}$ for all $i \ne 1$. For $\text{PSS}(L)$,  $\tilde{H}_2(w, L, L) = L / \Delta_{1, 2}^2$, and the failure probability bound in \eqref{eq:delta_bound} specializes to $$O \bigg(L\exp\bigg[- T \cdot \frac{\Delta^2_{1, 2} }{  192\cdot L } \bigg]\bigg),$$
which matches \eqref{eq:UP_loose_bound} up to   multiplicative factors in the exponent and the $O(\cdot)$ notation. 


\begin{table}[t]
  \centering
  \caption{ Comparison of PSS$(u)$ to Other Algorithms
  }
  \resizebox{ .48 \textwidth}{!}{ 
    \renewcommand{\arraystretch}{2.5}
	\begin{tabular}{ l | l |  l }
        Algorithm     &  Order of $  \gapbd$ & Order of $ \errpr  $   \\
	  \thickhline  
	  %
	  PSS$(u)$ & $   \displaystyle  \frac{  C    \log_u L   }{  T }  $  & $  \displaystyle  L (\log_u L)  \exp \bigg[ - \frac{ T}{  192    \tilH_2 (w,L,u)  \log_u L          } \bigg]
	  \vphantom{   \Bigg[      }
	    $  \\
	  \hline 
	  %
	  PSS$(2)$ & $   \displaystyle  \frac{   C   \log_2 L  }{  T }  $   & $  \displaystyle  L (\log_2 L ) \exp \bigg[ - \frac{ T}{   192  \tilH_2 (w,L,2)  \log_2 L  }   \bigg]  $ \\
	  %
	  SH & $   \displaystyle  \frac{   CL  \log_2 L  }{  T }  $    & $  \displaystyle  L (\log_2 L ) \exp \bigg[ - \frac{ T}{ 192    \tilH_2 (w,L,2)  \log_2 L  }   \bigg]  $ \\
	  %
	  PSS$(L)$ & $   \displaystyle  \frac{  C  }{  T }  $   & $  \displaystyle L  \exp \bigg(  -  \frac{    T    }{    192    L  / \Delta_{1,2}^2  }  \bigg)  $ \\
	  %
	  UP & $   \displaystyle  \frac{  CL }{  T }  $    & $  \displaystyle L  \exp \bigg(  -  \frac{    T    }{     192  L  / \Delta_{1,2}^2  }  \bigg)  $ 
	    \end{tabular}%
	    \renewcommand{\arraystretch}{1}
	    }
	  \label{tab:alg_comp}%
\end{table}

\noindent {\bf Comparisons in the corrupted setting.} Though the SH and the UP algorithms
can be directly applied to the setting with corruptions, we propose PSS$(u)$ to inject randomness in order to mitigate the impact of corruptions.
Intuitively, for an adversary with the knowledge of the algorithm,
the fact that a deterministic algorithm such as SH or UP pulls each active item 
 according to a deterministic schedule fixed at the start of a phase 
allows the adversary to corrupt rewards of the items to be pulled. 
However, PSS$(u)$ pulls items {\em probabilistically}, which prevents the adversary from identifying the items to be pulled even when the semantics of the algorithm are known to the adversary. 

 We analyse SH and UP using a similar analysis to our proof of Theorem \ref{thm:ub_probaSS}, and we tabulate the $(\gapbd, \errpr)$-PAC performance guarantee in Table \ref{tab:alg_comp}. While SH and UP have similar performance guarantees on $\errpr$ compared to their randomized counterparts, namely $\text{PSS}(u)$, the upper bounds on $\gapbd$ for SH and UP are {\em larger} than their randomized counterparts by a multiplicative factor of $L$. Consequently, the randomization in PSS$(u)$ allows us to mitigate the adversarial corruptions and leads to a better performance guarantee on $\epsilon_C$ compared to its deterministic counterparts.

\subsection{Lower bounds}
\label{sec:main_result_lb}


 In the previous section, we observed that the performance guarantee of $\text{PSS}(u)$ on $\epsilon_C$ deteriorates as the CPS increases. Interestingly, the deterioration is, in fact, fundamental to any online algorithm. Here, we demonstrate that no online algorithm is able to identify the optimal item with vanishing failure probability when  $C/T$ is above a certain threshold.  
 In fact, one attack strategy we design is shown to cause SH to fail miserably; in contrast, PSS$(2)$ remains robust to it.
 The impossibility result is further generalized to the identification of an $\epsilon$-optimal item for any $\epsilon\in [0, \Delta_{1,L} )$.  

 

\noindent {\bf Bernoulli instance.} 
We focus on instances where each item $i\in [L]$ follows $ \mathrm{Bern}(w(i))$, and $1 \! > \! w(1) \! > \!  w(2 ) \! \ge \! w(3) \! \ge \! \ldots \! \ge \! w(L) \! > \! 0$.
For any $\epsilon\in (0,1) $, we use
$
    L_{ \epsilon  } : = | \{ i\in [L] : \Delta_{1, i } \le \epsilon \} | 
$
to count the number of items with mean reward at most $\epsilon$ worse than that of the optimal item. 

\noindent{\bf Corruption strategy against general BAI algorithms.} Abbreviate $ \Delta_{1,2}$ as $\Delta$. Assume that $w(2) - \Delta > w(3)$. In this strategy, essentially, the adversary solely corrupts the reward of item~$1$, so that $\tilde{W}_t(1)\sim\text{Bern}(w(2) - \Delta)$, different from $W_t(1) \sim \text{Bern}(w(1))$, as long as there is enough corruption budget (Figure~\ref{fig:instance_lb_wc}). We  describe the corruption strategy in full in Appendix~\ref{pf:thm_bai_lb}.  
\vspace{-.15in}
    \begin{figure}[h]
	\centering
	\includegraphics[width=.48\textwidth, trim= 3 10 10 5,clip]{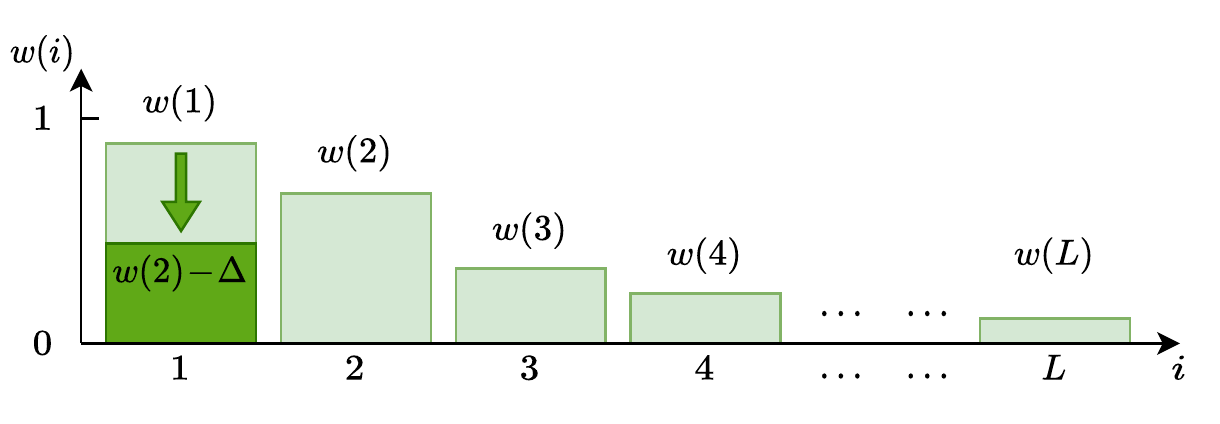}	
	\vspace{-.2in}
	\caption{ 
	    $\tilde{W}_t(1)\sim\text{Bern}(w(2) - \Delta)$
	    }
	\label{fig:instance_lb_wc}
\end{figure} \vspace{-.13in}


For a BAI with adversarial corruptions instance, we say that the instance has an optimality gap $\Delta > 0$ if $\Delta = \Delta_{1,2} > 0$. 
\begin{restatable}{theorem}{thmcoupling}\label{thm:bai_lb}
Fix   $\lambda\in (0, 1)$ and $ \Delta\in (0, 1/2)$.
 For any online algorithm, there is a BAI with an adversarial corruption  instance in $T$   steps, corruption budget $C = 1+(1+ \lambda)2\Delta T$, and optimality gap $\Delta$, such that 
\begin{align}
\bbP[ \Delta_{1, \iout } > 0 ]  &= \bbP[ \Delta_{1, \iout}  \ge \Delta] = \bbP[  \iout  \ne 1 ]  \nonumber\\*
&  \geq \frac{1}{2}\cdot \bigg[  1 - \exp\Big(-\frac{2\lambda^2 \Delta T }{3} \Big) \bigg]. \nonumber
\end{align}
\end{restatable}

 In particular, Theorem~\ref{thm:bai_lb} implies that, if the CPS satisfies
\begin{equation}\label{eq:CPS_bai_lbd}
C/T > 2\Delta_{1, 2},
\end{equation}
then it is impossible to identify the best item with probability $1 - \exp[-\Theta(T)]$. The upper bound in \eqref{eq:probaSS_C_OK} and the lower bound in \eqref{eq:CPS_bai_lbd} differ  by a multiplicative factor of $16 \lfloor \log_u L\rfloor$. Consequently, the upper bound in \eqref{eq:probaSS_C_OK} is within a factor of $O(\log_u L)$ away from the largest possible upper bound on CPS $C/T$, under which it is possible to identify the best item with probability at least $1 - \text{exp}[-\Theta(T)]$.

\noindent \textbf{Robustness of PSS$(2)$ with respect to SH.} 
Consider Theorem~\ref{thm:bai_lb}'s attack strategy (see Figure~\ref{fig:instance_lb_wc} and Appendix~\ref{pf:thm_bai_lb}), but applied to SH only in phase~1. We can show that SH will fail to identify the best item with probability at least~$1/2$.

\begin{restatable}{theorem}{thmcouplingSH}\label{thm:bai_lb_sh}
Fix  $L>1$, $\lambda\in (0, 1)$ and $ \Delta\in (0, 1/4)$.
 For the SH algorithm, there is a BAI with   adversarial corruption  instance with $T$ time  steps, corruption budget $C = (1+\lambda)2\Delta T / (L\log_2 L) $, and optimality gap $\Delta$, such that  if $T$ is sufficiently large,
\begin{align}
 \bbP[ \Delta_{1, \iout } > 0 ]  =\bbP[ \Delta_{1, \iout}  \ge \Delta]  = \bbP[  \iout  \ne 1 ]  \ge 1/2. \nonumber
\end{align}
\end{restatable}
Consequently, if $C/T \ge  \Delta_{1,2} / (L\log_2 L)$, SH fails to identify the best item with probability at least $1/2$ for large $T$. In contrast, PSS(2) identifies the best item with probability at least $1-\exp(-\Theta(T))$ as long as $C/T = O(\Delta_{1, 2} / \log_2 L)$ (see (\ref{eq:probaSS_C_OK})). Lastly, according to Table~\ref{tab:alg_comp}, SH would succeed with high probability if $C/T \le O(\Delta_{1,L} / (L\log_2 L))$. 
Hence, the upper and lower bounds of the CPS for SH 
are tight, {\em even up to log factors in $L$}.

The failure of SH is due to the fact that  according to the observation history, the adversary knows the item to pull at each time step.
In contrast, for PSS$(2)$, when determining $\{c_{t}(i)\}_{i\in [L]}$,  {\emph{the adversary only knows $\{W_{t}(i)\}_{i\in [L]}$, but does not know $i_t$.}} Rather, he only knows the {\em  distribution} of $i_t$. This uniform distribution facilitates exploration, while minimizing the leakage 
of the identity of $i_t$ to the adversary; this leads to an improvement by a factor of $\widetilde{O}(L)$ on  $\epsilon_C$.

\noindent{\bf Corruption strategy against identifying an $\epsilon$-optimal item.} 
We extend the previous strategy in order to impede the identification of an $\gapbd$-optimal item for any $\gapbd\in [0, \Delta_{1,L})$. Consider the following  two offline strategies:
\begin{itemize}[ itemsep = -15pt,   topsep = -1pt]
     \item[(I)] at each time step, if the random reward is $1$, the adversary shifts it to $0$ until the corruption amount is depleted (see Figure~\ref{fig:instance_lb});
     \item[] \vspace{1em}
    \begin{figure}[h]
	\centering
	\includegraphics[width=.48\textwidth, trim= 12 10 10 3,clip]{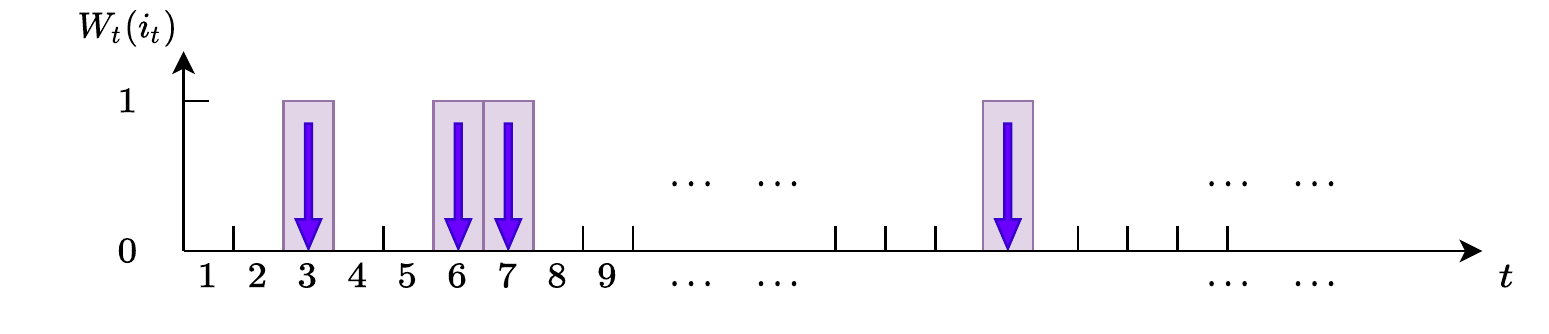}	
	\vspace{-.15in}
	\caption{ 
	    Shift $W_t ( i_t)$ to $  0 $ When $W_t ( i_t  )=  1$ 
	    }
	\label{fig:instance_lb}
\end{figure}
\vspace{-.15in}
    \item[(II)] at each time step, if the random reward is $0$, the adversary shifts it to $1$ until the corruption amount is depleted.
\end{itemize}
The design of either strategy aims  to make the agent obtain the same random reward at all time steps.
As a result, the agent fails to get any information from the observations. 
In this case, the best thing she can do is to output any item with a uniform probability of $1/L$ after $T$ time steps.

\begin{restatable}{theorem}{thmLbBerShiftToZeroOne}
\label{thm:lb_ber_shift_to_zero_one}
Fix  any  $\lambda, \epsilon \in (0,1)$.
If  $C\ge  L \cdot \{ 1 - (1 - \lambda )[ 1-w(1) ]  \}\cdot T$,
 Strategy (I)'s attack results in
\begin{align*}
    \bbP [ \Delta_{1,\iout} > \epsilon ]  \ge 1 -  \frac{ L_{\epsilon  } }{L} -  \exp\bigg[ - \frac{ \lambda^2 T L [ 1- w(1) ]  }{2}  \bigg]  .
\end{align*}
If instead $C \ge  L \cdot [1 - (1-\lambda ) w(L) ]\cdot T$, 
Strategy (II)'s attack results in
\begin{align*}
    \bbP [ \Delta_{1,\iout} > \epsilon ]  \ge 1 -  \frac{ L_{ \epsilon } }{L} -  \exp\bigg[ - \frac{ \lambda^2 T L  w(L) }{2}  \bigg].
\end{align*}

\end{restatable}

%
%
%
%

When $\epsilon< \Delta_{1, 2} $ so that $L_\epsilon=1$,    Theorem~\ref{thm:lb_ber_shift_to_zero_one}
provides   lower bounds for the probability of identifying the optimal item under   corruption
strategies (I) and (II) respectively.  
In this case, when $T\rightarrow \infty $,  the  failure probability is asymptotically lower bounded by $1-1/L$.

Although the adversary can use adaptive strategies to attack random rewards, i.e., design a strategy to add corruptions according to past observations,
Theorem~\ref{thm:lb_ber_shift_to_zero_one}  
 shows that when the corruption is sufficiently large, even an {\em offline} strategy, i.e., one that is {\em fixed} before the algorithm runs,   prevents the agent from identifying a satisfactory item with high probability. Thus, if $C= \Omega(T)$, any algorithm will fail to identify a near-optimal item with high probability.
Therefore, PSS$(u)$ is tight up to a factor that differs from $\log_u L$ in Theorem~\ref{thm:ub_probaSS}  to $L$ in Theorem~\ref{thm:lb_ber_shift_to_zero_one}.

%
%

\section{Proof Sketch of Theorem~\ref{thm:ub_probaSS} }

\label{sec:pf_sketch_ub_probaSS}
We provide the proof sketch for Theorem~\ref{thm:ub_probaSS}. The detailed proof and those of Theorem~\ref{thm:bai_lb}--\ref{thm:lb_ber_shift_to_zero_one} are deferred to the supplementary material.

\noindent \textbf{Feasibility.} 
We first show that our algorithm is feasible in the sense that the $M$ phases proceed within $T$   steps, and $A_M$ is a singleton.
\begin{restatable}{lemma}{lemmaUbprobaSSSingleOutput}
  \label{lemma:ub_probaSS_single_output}
    It holds that $ NM \le T$ and $| A_M | = 1$.
\end{restatable}
Lemma~\ref{lemma:ub_probaSS_single_output} ensures that $\iout$ is well-defined.
Moreover, it implies that 
$\{ 1 \ne \iout \} \! = \! \{ 1 \notin A_M \}.$

\noindent \textbf{Concentration.} 
 At the end of phase $m~(1\le m \le M)$, the agent shrinks the active set $A_{m-1}$ according to the $\hat{w}_m(i)$'s, the corrupted empirical means of the active items.
Intuitively, we expect that if $\hat{w}_m(i)$ and $w(i)$ are sufficiently close, we can identify item $i$ with small $\Delta_{1,i}$. 
To this end, we define the amount of corruptions during phase $m$ as
\begin{align*}
    C_m := \sum_{ t = T_{m-1}  +1 }^{ T_{m } } \max_{ i\in [L] } | c_t(i) |.
\end{align*}
To estimate the gap between $\hat{w}_m(i)$ and $w(i)$, we define a class of ``nice events'' for all $i \in A_{m-1}$ and $ a   \in ( 0,1)$:
\begin{align*}
    & \calE_{m,i  }^{ (\rmU) } (a) :=  \Big\{   \hat{w}_m(i)  <  w(i)  +   \frac{2  C_m }{ N  } +  2 a 
     \Big\}, \\*
    & \calE_{m,i  }^{ (\rmL) } (a):=  \Big\{  \hat{w}_m(i)  > w(i)  -   \frac{2 C_m }{ N  } -  2 a 
    \Big\}. 
\end{align*}
We utilize Theorem~\ref{thm:conc_var_chernoff} and \ref{thm:conc_martingale} to show that all these events hold with high probability.
In particular, Theorem~\ref{thm:conc_martingale} allows us to bound the impact of corruptions.

\begin{restatable}{lemma}{lemmaUbprobaSSConcApply}
\label{lemma:ub_probaSS_conc_apply}
Let $\overline{\calE} $ denote the complement of any event $\calE$.    For any fixed $m, i\in A_{m-1} $ and  $a    \in ( 0,1)$, 
        \begin{align*}
        & 
        \resizebox{.48\textwidth}{!}{$      
        \bbP \big[   \overline{  \calE_{m,i  }^{ (\rmU) }    (a)  }   \big]   \! \le  \!  2\exp  \! \Big[  \! -\! \frac{  a ^2  \! \cdot   \!  n_m  }{3} \Big]    , \
        \bbP \big[   \overline{  \calE_{m,i  }^{ (\rmL) }  (a)   }   \big]  \!  \le \!    2 \exp \!  \Big[ \! - \! \frac{   a^2 \!  \cdot  \!  n_m  }{3} \Big] .
        $}
    \end{align*} 
\end{restatable}
Note that $n_m$ is the expected number of pulls of each active item $i\in A_{m-1}$ during phase $m$.
Lemma~\ref{lemma:ub_probaSS_conc_apply} implies that we are able to bound the gap between $\hat{w}_m(i)$ and $w(i)$ for each active item $i \in A_{m-1}$ with high probability.

\noindent \textbf{Technique.} In light of the importance of randomization for the regret minimization problem \citep{lykouris2018stochastic,gupta2019better,zimmert2018tsallisinf},
we inject randomness in PSS$(u)$ 
and derive Lemma~\ref{lemma:ub_probaSS_conc_apply}, which 
explains the necessity of Line~\ref{line:probaSS_prob_pull} in Algorithm~\ref{alg:probaSS} in order to mitigate the impact of adversarial corruptions.
While an active item is pulled probabilistically in PSS($u$), it is pulled for a {\em fixed} number of times in SH.
Though the expected number of pulls of one active item is of the same order for PSS$(2)$ and SH, 
the absence of randomization in SH does not allow Theorem~\ref{thm:conc_martingale} to bound the gap between $\hat{w}_m(i)$ and $w(i)$ in the same way as for PSS$(2)$. 
For SH, we can only show that
\begin{align*}
    & 
    \bbP \Big[   \hat{w}_m(i) \!   <  \!  w(i)  +   \frac{C_m |A_{m-1}| }{ N  } +  a 
     \Big] \le \exp \Big[  -\frac{  a ^2 \cdot   n_m  }{3} \Big] ,
\end{align*}
and similarly for the upper tail. 
Disregarding constants, the difference between these bounds and those for PSS$(2)$ in Lemma~\ref{lemma:ub_probaSS_conc_apply} is that the term involving $C_m$ is worse by a factor of $|A_{m-1}|$ for SH.
As a result, the bound on $\Delta_{1, \iout}$ for SH turns out to be $O (CL \log_2 L/T )$, which is worse than that for PSS$(2)$ by a factor of $L$ (see Table~\ref{tab:alg_comp}).
A similar explanation is also applicable to explain the difference between the bounds for UP and PSS$(L)$.

\noindent \textbf{Elimination of the optimal item.}
When the agent fails to output item $1$ (the optimal item), i.e.,$1\ne \iout$, 
item $1$ is inactive by the end of the last phase of the algorithm. Let
$
    m_1 := \min \{ m  \in [M]: 1\notin A_m \}$, 
where $\min \emptyset = \infty$.
Since $1 \ne \iout$, we have $m_1\le M.$ 
 The index $m_1$ labels the phase during which item $1$ turns from active to inactive.
 Next, any item  $i $ that belongs to the active set $A_{m_1}$ satisfies 
$w(i) < w(1)$  and $\hat{w}_{ m_1 }(i) \ge \hat{w}_{m_1} (1)$. 
Conditioning on $\calE_{m_1,1  } ^{ (\rmL) } (a)$ and $\calE_{m_1 ,i }^{ (\rmU) }  (a)$,  we have 
 \begin{align*}
     \resizebox{ .48\textwidth }{!}{$
     ( - \infty, w(i)  +   \frac{ 2C_m  }{ N  } +  2a  ]\cap
     [ w(1)  -  \frac{ 2C_m }{ N  } - 2a   , + \infty)\neq \emptyset.
     $}
 \end{align*} 
%
 To facilitate our analysis, we set $a_i: = \Delta_{1,i}/8 $ for all $2\le i \le L.$ 
We let $j_1$ be the item in $A_{m_1}$ with the smallest mean reward, i.e., 
$
   j_1 := \argmin_{ i \in A_{ m_1}  } w(i)$.
Lemma~\ref{lemma:ub_probaSS_conc_apply} implies that with  probability
$1 - 4 \exp ( - \Delta_{1,j_1}^2 \cdot  n_{m_1}   / { 192 } ) ,$
 we have 
 $   \Delta_{ 1, j_1 } \le { 8 C_{ m_1 } }/{ N } $.
Since $\iout \in A_{m_1}$, we have $w(j_1) \le w(\iout)$. This allows us to bound $\Delta_{1, \iout}$ as follows:
\begin{align*}
    \Delta_{1 ,\iout } \le \Delta_{ 1,j_1 } \le   \frac{ 8 C_{ m_1 } }{ N } \le  \frac{ 8 C }{ N }.
\end{align*}

Note that $m_1$, $j_1$ are random variables that depend on the dynamics of the algorithm.
For any realization of $m_1$, $j_1$,
we formulate the observation above in Lemma~\ref{lemma:ub_probaSS_conf_apply_gap}. 
The complete proof of Lemma~\ref{lemma:ub_probaSS_conf_apply_gap} is postponed to Section~\ref{pf:lemma_ub_probaSS_conf_apply_gap}.

\begin{restatable}{lemma}{lemmaUbprobaSSConfApplyGap}
\label{lemma:ub_probaSS_conf_apply_gap}

 Conditioned on $\calE_{ m,1 }^{ (\rmL) }  (a_i) $ and $\calE_{ m,i  }^{ (\rmU) } (a_i) $, where $a_i = \Delta_{1,i}/8$ for each $2\le i \le L,$ we have 
\begin{align*}
     \{ 1   \in     A_{m-1},     1    \notin   A_m,      i  \in A_m \}
    \subset
    \Big\{ 
        \Delta_{1,i}  \le \frac{  8 C_m}{   N  }  
    \Big\}.
\end{align*} 
\end{restatable}

 \noindent \textbf{Bounds.}
%
When $\calE_{ m_1,1  }^{ (\rmL) } ( a_{ j_1 } )$ and $\calE_{ m_1, j_1  }^{  ( \rmU ) } ( {a}_{ j_1 } ) $ hold, we can apply Lemma~\ref{lemma:ub_probaSS_conf_apply_gap}  to bound $\Delta_{1,\iout}$ with the total corruption budget $C$,
i.e., for any realization of $m_1$, $j_1$,
\begin{align}
    &
    \bbP \Big[   \Delta_{ 1 , \iout } > \frac{ 8 C }{   N } \Big] 
    \le
    \bbP \Big[ 
    ~ \overline{   \calE_{ m_1,1  }^{ (\rmL) } ( a_{ j_1 } )  \bigcap \calE_{ m_1, j_1  }^{  ( \rmU ) } ( {a}_{ j_1 } )    } ~ 
    \Big].
    \nonumber 
\end{align}
In addition, the definitions of $j_1$ and $A_m$ indicate that
\begin{align*}
    & j_1  \ge | A_{m_1} |   =   \bigg\lceil \frac{L}{ u^{m_1} } \bigg\rceil
     \ge   \frac{     \big\lceil \frac{L}{   u^{m_1-1} } \big\rceil   }{  u  }
     =   \frac{   | A_{m_1 -1 } | }{  u  },
\end{align*}
and     $| A_{ m_1 -1 } |   \le   L$.
 These inequalities, along with Lemma~\ref{lemma:ub_probaSS_conc_apply}, the definitions of $ {a}_{i }$, $N_m$ and $\tilde{H}_2(w,   L,  u),$ imply that for all $1 \le  m   \le    M$ and $ 2   \le   i    \le   L ,$
\begin{align}
    &
     \bbP \Big[  \Big( ~ \overline{   \calE_{ m_1,1  }^{ (\rmL) } ( a_{ j_1 } )  \bigcap \calE_{ m_1, j_1  }^{  ( \rmU ) } ( {a}_{ j_1 } )    }~ \Big) \bigcap \{ m_1=m, j_1 = i  \}   \Big]
    \nonumber \\ &
    \qquad\le 4 \exp \bigg[   - \frac{   N   }{  192  \tilH_2 (w,L,u) }   \bigg] 
    .
    \nonumber
\end{align}
Altogether,
\begin{align*}
    & \bbP\bigg[  \Delta_{ 1 , \iout } > \frac{ 8 C }{  N }    \bigg]
    \\&
    \le \sum_{m=1}^M \sum_{i=2}^L 
        \bbP\bigg[  \Big\{ \Delta_{ 1 , \iout } > \frac{ 8 C }{ N } \Big\} \bigcap \{ m_1=m, j_1 = i    \} \bigg]
    \\&
    \le 4M(L-1) \exp \bigg[   - \frac{     N   }{  192 \tilH_2 (w,L,u) }   \bigg].
\end{align*}
We complete the proof of Theorem~\ref{thm:ub_probaSS} with $N =  \lfloor T/M \rfloor$, $M = \lceil \log_u L \rceil$.
We elaborate on the details in Section~\ref{pf:ub_probaSS_thm_final_step}.

\section{Numerical Experiments and Conclusion}

We compare the performances of PSS($2$), SH and UP under the corruption strategy considered in Theorem~\ref{thm:bai_lb}. In the experiments, we set the mean of the optimal item to be $w^*\in \{ 0.4, 0.5 \}$, the mean of $L-2$ suboptimal items to be $w' = 0.2$. We set $\Delta=(w^*-w')/3$ and the mean of the remaining  item to be $w^* - \Delta$.  
%
The CPS  $C/T=(1+\lambda)2\Delta /(L \log_2 L)$ (cf.\ Theorem~\ref{thm:bai_lb_sh}).
For each algorithm and instance, we ran $100$ independent trials and report the percentage of trials each algorithm succeeds in identifying the optimal item.  Further experiments are provided in Appendix~\ref{appdix_exp}.  {The codes to reproduce all the experiments can be found at \url{https://github.com/zixinzh/2021-ICML.git}}.

\begin{figure}[t]
    \begin{flushright}
    \includegraphics[width = .28\textwidth ]{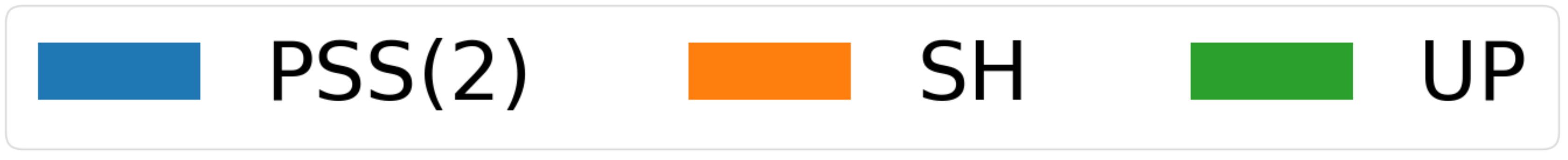} 
    \hphantom{aa}
    %
      \end{flushright}
   \vspace{-1.3em}
  \centering 
  \subfigure[$ \lambda=0.5$]{
\label{pic:exp_compare_w_lambda_a}
\includegraphics[width = .23\textwidth]{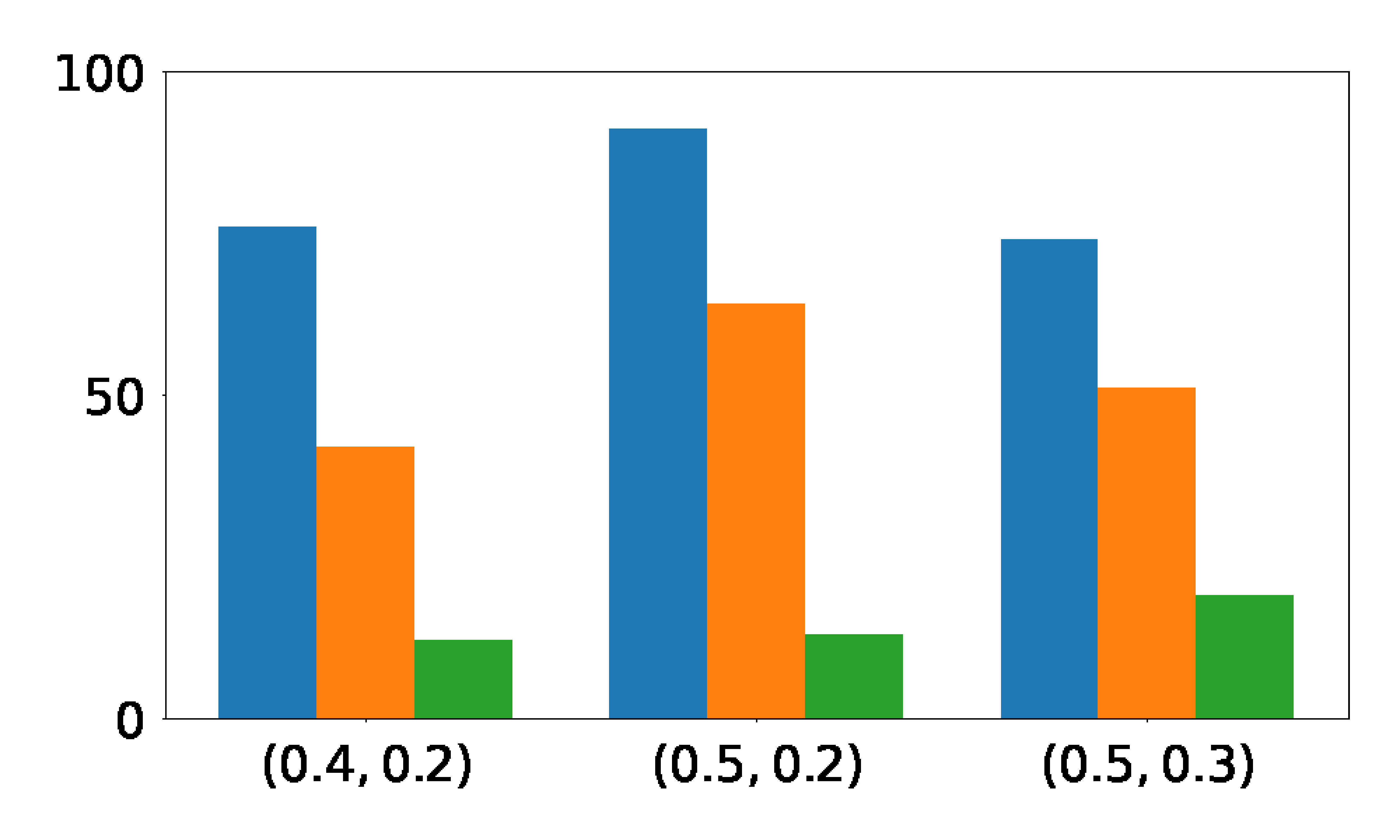} 
}
\hspace{-.5em}
\subfigure[$ \lambda=0.9$]{
\label{pic:exp_compare_w_lambda_b}
\includegraphics[width = .23\textwidth]{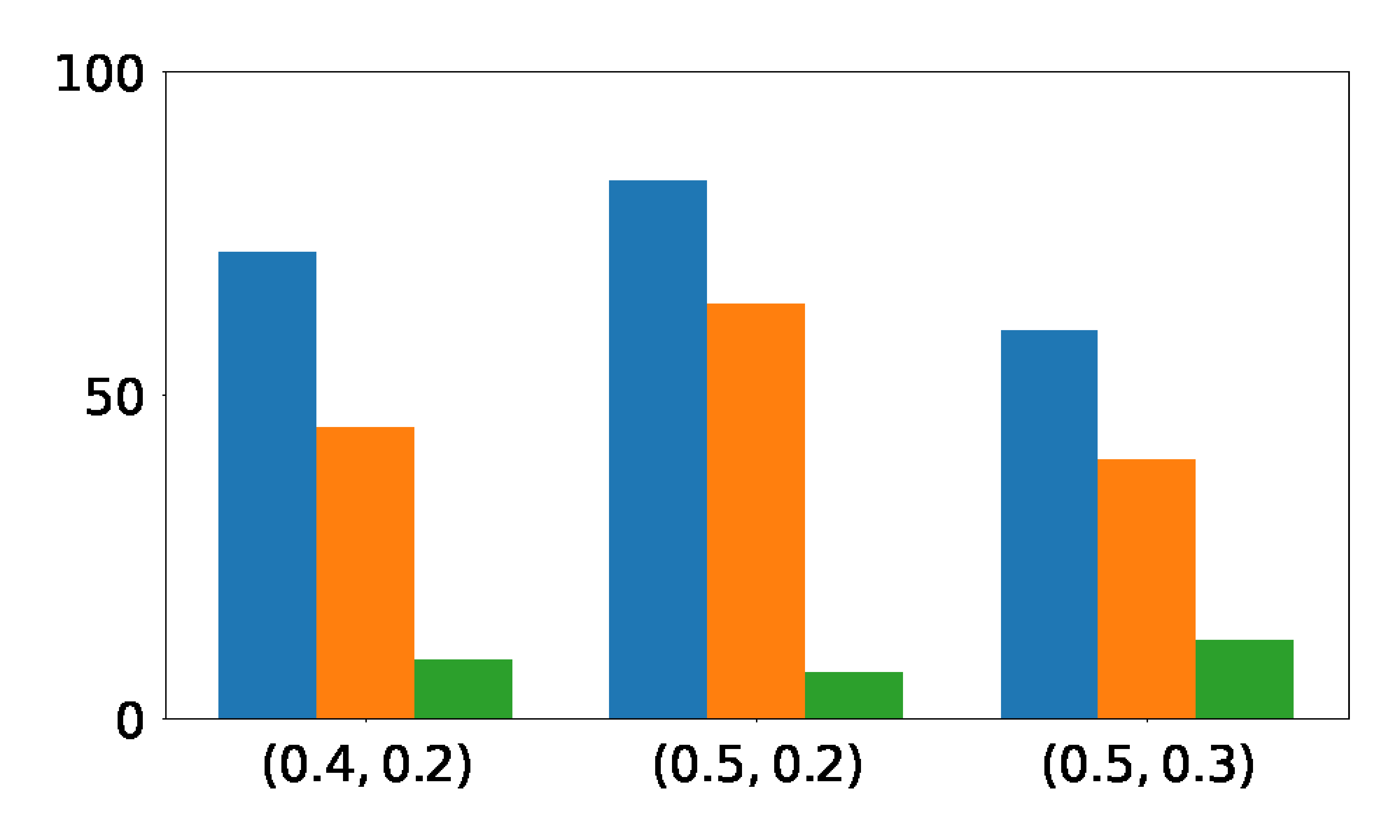} 
}
\\
\vspace{-.3em}
  \subfigure[Effect of $L$ ($T=2\times 10^3$)]{
\label{pic:exp_compare_L}
\includegraphics[width = .23\textwidth]{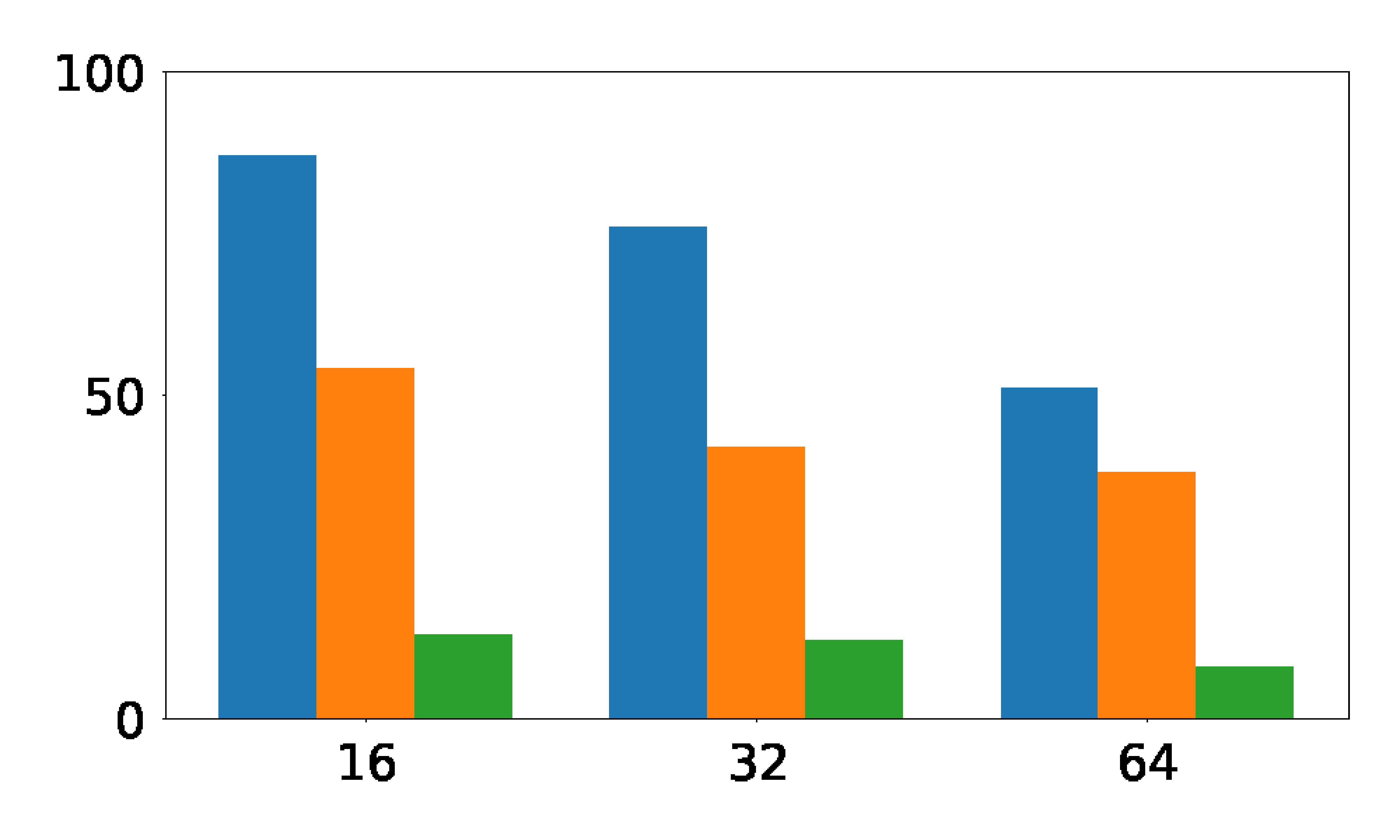}  
}
\hspace{-.5em}
\subfigure[Effect of $T$ ($L=32$)]{
\label{pic:exp_compare_T}
\includegraphics[width = .23\textwidth]{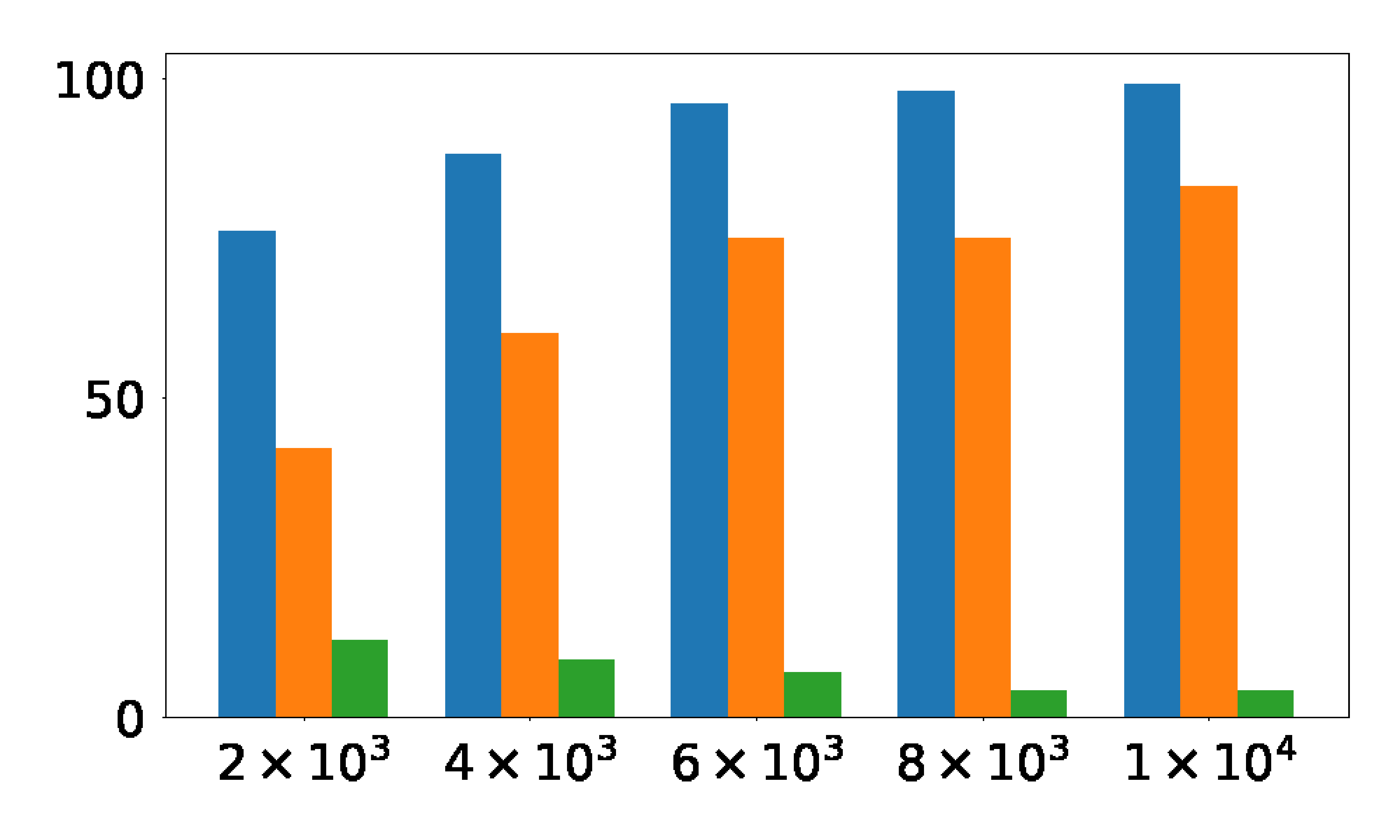} 
}\vspace{-.05in}
  \caption{Percentage of correct BAI of PSS($2$), SH and UP. We fix $T=2\times 10^3, L=32$ and vary $w^*,w'$ in (a) and (b). We fix $\lambda=0.5, w^*=0.4, w'=0.2$  in (c) and (d).   }
  \label{pic:exp_compare}
\end{figure} 
Overall, Figure~\ref{pic:exp_compare} implies that PSS$(2)$ always outperforms SH and UP for a BAI problem that is attacked by the strategy of Theorem \ref{thm:bai_lb}, underscoring the importance of randomization. 
Next, we observe from  Figures~\ref{pic:exp_compare_w_lambda_a}-\ref{pic:exp_compare_w_lambda_b} that a larger $\Delta$ means that the difference between the optimal and suboptimal items is more pronounced, resulting in better performances across all algorithms, even if the CPS increases. 
Since the CPS increases with $\lambda$, each algorithm identifies the best item less often when $\lambda$ increases (see Table~\ref{tab:exp_compare_w_lambda}). 
{Figure \ref{pic:exp_compare_L} shows that 
the agent identifies the best item less often  when $L$ increases.} This implies that even if we let  the CPS decrease with $L$ (per Theorem \ref{thm:bai_lb_sh}), the larger size of the ground set still makes the instance more difficult.
Lastly, 
Figure~\ref{pic:exp_compare_T} 
shows that when the CPS is fixed, a larger $T$ increases the success probabilities of PSS($2$) and SH. 


\textbf{Summary and Future Work.} 
This paper has deepened our understanding the fundamental performance limits
of BAI algorithms in their  ability to cope with  adversarial corruptions that are added on to the random rewards. We designed  {\sc PSS}$(u)$, an algorithm  that can be regarded as a robustification of the {\sc SH} algorithm by \citet{karnin2013almost}.  Due to {\sc PSS}$(u)$'s inherent randomized nature, it can successfully mitigate the adversarial corruptions. Furthermore, we showed by way of constructing several adversarial corruption strategies that the optimality gap of {\sc PSS}$(u)$ is $O(\log L)$-competitive vis-\`a-vis {\em any} corruption-tolerant algorithm. 
These attack strategies are shown to break SH but PSS$(u)$ remains robust to the corruptions.    

Inspired by 
\citet{liu2019data}, \citet{jun2018adversarial}, and \citet{zuo2020near}, it would be fruitful to devise optimal corruption strategies for algorithm-specific and algorithm-independent settings to uncover whether the dependence of the smallest optimality gap $\gapbd$ on $\log L$ is fundamental. We conjecture that the smallest $\gapbd$ does not depend on $L$. More ambitiously, we would like to close the gap between the upper and lower bounds in \eqref{eq:probaSS_C_OK} and \eqref{eq:CPS_bai_lbd}.

\section*{Acknowledgements}
This work is partially funded by a
National University of Singapore Start-Up Grant (R-266-000-136-133),
a Singapore National Research Foundation (NRF) Fellowship (R-263-000-D02-281)
and
a Singapore Ministry of Education AcRF Tier 1 Grant (R-263-000-E80-114).

\newpage
\vphantom{a}
\newpage
\bibliography{corruptionBai_ref}
\bibliographystyle{icml2020}


\newpage
\onecolumn

\appendix

\setcounter{table}{0}
\renewcommand{\thetable}{A.\arabic{table}}

\setcounter{figure}{0}
\renewcommand{\thefigure}{A.\arabic{figure}}

\section{Notations}

\begin{spacing}{1.42}
	\begin{longtable}[!h]{ p{.15\textwidth}  p{.85\textwidth} } 
           $[n]$ & set $\{1,\cdots,n\}$ for any $n\in \bbN$ \\
		   $[L]$ & ground set of size $L$ \\
		   $\nu(i)$ & reward distribution of item $i \in [L]$ \\
		   $w(i)$ & mean reward of item $i \in [L]$ \\
		   
		   $W_t(i)$ & random reward of item $i$ at time step $t$\\
		   $c_t(i)$ & corruption added on random reward item $i$ at time step $t$\\
		   $\tilde{W}_t(i)$ & corrupted reward of item $i$ at time step $t$\\
		   $i_t$ & pulled item at time step $t$\\
		   $C$ & total corruption budget \\
		   
		   $\bbP$ & probability law of the process $ \{ \tilde{\mathbf{W} }_t = ( \tilde{W}_t(1) , \ldots,  \tilde{W}_t(L)  ) \} ^T_{t=1}$ \\
		   $\Delta_{1,i}$ & gap between mean rewards of item $1$ and $i$  \\
		   $\epsilon$ & optimality gap of item  \\
		   
		   $\pi$ & non-anticipatory algorithm \\
		   $i_t^\pi$ & pulled item of algorithm $\pi$ at time step $t$\\
		   $\iout^{\pi,T}$ & output of algorithm $\pi$ \\
		   $\phi^{\pi,T}$ & final recommendation rule of algorithm $\pi$\\
		   $\calF_t$ & observation history \\ 
		   
		   $\epsilon_C$ & bound on $\Delta_{ 1, \iout^{\pi, T} }$ \\ 
		   $\delta$ & failure probability \\
		   
		   $u$ & parameter in Algorithm~\ref{alg:probaSS} \\
		   $M$ & amount of phases in Algorithm~\ref{alg:probaSS} \\
		   $N$& length of one phase in Algorithm~\ref{alg:probaSS} \\
		   $A_m$ & active set in Algorithm~\ref{alg:probaSS} \\
		   $q_m$ & probability to pull an active item during phase $m$ in Algorithm~\ref{alg:probaSS} \\
		   $n_m$ & expected number of pulls of an active item during phase $m$ in Algorithm~\ref{alg:probaSS} \\
		   $\hat{w}_m(i)$ & corrupted empirical mean of item $i$ during phase $m$ in Algorithm~\ref{alg:probaSS}  \\

		   $\tilde{H}_2( w,L,u )$ & difficulty of the instance $\{ w(i) \}_{i=1}^L$ for PSS$(u)$ \\
		   $H_2(w)$ & intrinsic difficulty of the instance $\{ w(i) \}_{i=1}^L$  \\
		   $\mathrm{Bern}(a)$ & Bernoulli distribution with parameter $a\in[0,1]$ \\
		   $L_\epsilon$ & number of item $i$ with $\Delta_{1,i} \le \epsilon $  \\
		   $ \Delta$ & equals to $\Delta_{1,2}$ \\
		   $\lambda$ & parameter in the analysis of corruption strategies \\

		   $ C_m $ & amount of corruptions during phase $m $  \\
		   $\mathcal{E}_{m,i}^{ (\mathrm{U}) }(a), \mathcal{E}_{m,i}^{ (\mathrm{L}) }(a)$ & ``nice events'' in the analysis of Algorithm~\ref{alg:probaSS} \\
		   $m_1$ & index of the phase during which item $1$ turns from active to inactive  \\
		   $j_1$ & item in $A_{m_1}$ with the least mean reward \\
		   $a_i$ & equals to $ \Delta_{1,i}  / 8$ for all items $ 2 \le i \le L$ 
		   
 	   
	\end{longtable}

\end{spacing}
	\addtocounter{table}{-1}

\section{Useful theorems}

\label{sec:useful_thm}


\begin{theorem}[Standard multiplicative variant of the Chernoff-Hoeffding bound; 
\citet{dubhashi2009concentration}, Theorem 1.1] \label{thm:conc_var_chernoff}
     Suppose that $X_1, \ldots  , X_T$ are independent $[0, 1]$-valued random variables, and let $X = \sum^T_{t=1} X_t$. Then for any $\epsilon > 0$,
     \begin{align*}
         \Pr [ ~        X - \mathbb{E} X \ge \epsilon \mathbb{E} X  ~ ]    \le   \exp \left(     - \frac{ \epsilon^2 }{3}  \mathbb{E} X  \right), \quad
         \Pr [ ~        X - \mathbb{E} X  \le -\epsilon \mathbb{E} X  ~ ]    \le  \exp \left(     - \frac{ \epsilon^2 }{3}  \mathbb{E} X  \right).
     \end{align*}

\end{theorem}



\begin{theorem}[\citet{beygelzimer2011contextual}, Theorem 1; \citet{gupta2019better}, Theorem 10]   \label{thm:conc_martingale}
    Suppose that $X_1, \ldots , X_T$ is a martingale difference sequence with respect to a filtration $ \{ \mathcal{F}_t \}^T_{t=1}$, and let $X = \sum^T_{t=1} X_t$. Assume that $|X_t| \le b $ for all $t$, and define $V =
\sum^T_{t=1}  \mathbb{E} [ X^2_t  |  \mathcal{F}_{t–1}]$. Then for any $\delta > 0$,
\begin{align*}
    \Pr \left[      X \le \frac{V}{b} + b \ln \frac{1}{\delta}   \right] \ge 1 - \delta.
\end{align*}
\end{theorem}



\begin{theorem} [Multiplicative Chernoff Bound~\citep{mitzenmacher2017probability,chen2016combinatorial}]
\label{thm:conc_bernoulli}
	Let $X_{1}, \cdots, X_{n}$ be Bernoulli random variables taking values in $\{0,1\}$ such that $\mathbb{E}\left[X_{t} | X_{1}, \cdots, X_{t-1}\right] \geq \mu$ for all $t \leq n,$ and $Y=X_{1}+\ldots+X_{n}$. Then, for all $\delta \in(0,1)$
	$$\operatorname{Pr}[Y \leq(1-\delta) n \mu] \leq e^{-\frac{\delta^{2} n \mu}{2}}.$$
\end{theorem}

\section{Proofs of main results}

In this section, we 
provide proofs of Lemmas~\ref{lemma:ub_probaSS_single_output}  -- \ref{lemma:ub_probaSS_conf_apply_gap},
complete the proof of Theorem~\ref{thm:ub_probaSS},
and provide the proofs of Theorem~\ref{thm:bai_lb} -- \ref{thm:lb_ber_shift_to_zero_one}.

\subsection{Proof of Lemma~\ref{lemma:ub_probaSS_single_output} }
\label{pf:lemma_ub_probaSS_single_output}

\lemmaUbprobaSSSingleOutput*
\begin{proof}

    (i) $NM = \lfloor T/M \rfloor \cdot M \le T/m \cdot M = T.$
    
    (ii) Since $M = \lceil \log_u L \rceil$, $ | A_M  |  = \big\lceil  { L }/{ u^M } \big\rceil$, we have
    \begin{align*}
        & L \le u^M
            ~\Rightarrow~
            | A_M  |   \le  \bigg\lceil \frac{ u^M }{ u^M } \bigg\rceil =1,  
    \quad \quad
           \frac{ L }{ u^{M-1} }  >0
        ~ \Rightarrow~  | A_M | \ge 1.
    \end{align*}
\end{proof}

\subsection{Proof of Lemma~\ref{lemma:ub_probaSS_conc_apply}}

\label{pf:lemma_ub_probaSS_conc_apply}

\lemmaUbprobaSSConcApply*

\begin{proof}

(i)
Let $Y_t(i)$ be an indicator for item $i$ being pulled at time step $t$ and $\tilde{n}_m(i)$ be the number of pulls of item $i$ during phase $m$. 
Recall that $W_t(i)$ is the stochastic reward of item $i$ at time step $t$ and $c_t(i) = \tildeW_t(i) - W_t(i)$ is the corruption added to this item by the adversary at this time step.
 Note that $c_t(i) $ may depend on all the stochastic rewards up to (and including) time step $t$, and also on all previous choices of the algorithm (though not the choice at step $t$). 
We denote $E_m := [ T_{m-1}+1, \ldots, T_m ]$ as the $N $ many time steps in phase $m$. 
Then
\begin{align*}
    \hatw_m(i) = \frac{1}{ n_m  }  \sum_{ t \in E_m}  Y_t(i)  [ W_t(i) + c_t(i) ].
\end{align*}
For ease of analysis, let us break the sum above into two, and define
\begin{align*}
    A_m(i) = \sum_{ t \in E_m}   Y_t(i) W_t(i)   \quad\mbox{and}\quad
    B_m(i) = \sum_{ t \in E_m}   Y_t(i) c_t(i) .
\end{align*}

(ii)
 Let us first bound the deviation of $A_m(i)$. Observe that $W_t(i)$ is an independent draw from a $[0, 1]$-valued r.v. with mean $w(i)$ and $Y_t(i)$ is an independent random variable drawn from $\{0,1\}$ with mean $q_m $.
Moreover, we have that $\bbE[A_m(i) ] = N  \cdot [q_m  w(i) ] = n_m  w(i) \le n_m $. Hence, for any $a_{1,m,i} >0$, a Chernoff-Hoeffding bound
(a multiplicative version thereof) as in Theorem~\ref{thm:conc_var_chernoff} implies that
    \begin{align*}
        \bbP \bigg[   \frac{ A_m(i) }{ n_m   }  - w(i)    \ge  a_{1,m,i}  ~\bigg] \le 
        \exp \bigg[  -\frac{ a_{1,m,i}^2 \cdot   n_m  }{3} \bigg] ,  \quad
        \bbP \bigg[   \frac{ A_m(i) }{ n_m  }  - w(i)    \le - a_{1,m,i}   ~\bigg] \le 
        \exp \bigg[  -\frac{ a_{1,m,i} ^2 \cdot   n_m  }{3} \bigg] .
    \end{align*}
%

(iii)
Next, we turn to bound the deviation of $B_m(i)$. Consider the sequence of r.v.s $X_1, \ldots , X_T$ defined by $X_t = [ Y_t(i)  – q_m  ] \cdot c_t(i)  $ for all $t$. Then $\{ X_t \}^T_{t=1}$ is a martingale difference sequence with respect to the filtration $ \{ \tilde{ \calF}_t \}^T_{t=1}$, where $$ \tilde{ \calF}_t = \sigma(  \{Y_s(i) \}_{ s\le t, i\in [L] }   ,  \{ W_s(i) \}_{ s\le t+1, i\in [L] }  , \{ c_s(i) \}_{ s\le t+1, i\in [L] } ).$$
 {According to the problem setup, the adversary obtains more information than the agent, which results in the difference between $\calF_t$ defined in Section~\ref{sec:prob_setup} and $\tilde{\calF}_t$ here.}
Since the corruption $c_t(i)$ becomes a deterministic value when conditioned on $\tilde{ \calF}_{t-1}$ (as we assume a deterministic adversary),
and since $\bbE[Y_t(i) |  \tilde{ \calF}_{t-1} ] = q_m $, we have
\begin{align*}
    \bbE [ X_t |  \tilde{ \calF}_{t-1} ] = \bbE[ Y_t(i) – q_m  |  \tilde{ \calF}_{t-1}  ] \cdot c_t(i)= 0.
\end{align*}
Further, we have $ |X_t |,  |c_t(i)|  \le  1$ for all $t$, and we can bound the predictable quadratic variation of this martingale as
\begin{align*}
      V  &= \sum_{ t \in E_m} \bbE[ X_t^2 | \tilde{ \calF}_{t-1} ]  
    = \sum_{t \in E_m } \bbE\big[~ [ Y_t(i)  – q_m  ]^2 |  \tilde{ \calF }_{t-1} \big]    \cdot  c_t(i)^2 
    \le \sum_{t \in E_m }  | c_t(i) |  \cdot \bbE\big[~ [ Y_t(i)  – q_m  ]^2 |  \tilde{ \calF}_{t-1} \big]  
    \\ &
    = \sum_{ t \in E_m} | c_t(i) |\cdot \var[ Y_t(i) ] 
    = \sum_{ t \in E_m} | c_t(i) |\cdot q_m \cdot(1 - q_m)
     \le  q_m  \cdot \sum_{ t \in E_m}  | c_t(i) |.
\end{align*} 

Applying a Freedman-type concentration inequality for martingales (Theorem~\ref{thm:conc_martingale}), we obtain that except with probability $ \delta_{2,m,i} $~(setting $b=1$ in Theorem~\ref{thm:conc_martingale}), 
\begin{align*}
    \frac{ B_m(i) }{ n_m } 
    \le  \frac{ q_m  }{ n_m  } \cdot \sum_{ t\in Em} | c_t(i) | + \frac{ V + \log( 1/ \delta_{2,m,i} ) }{ n_m } 
    \le \frac{ 2 q_m  }{ n_m } \cdot \sum_{ t\in Em} | c_t(i) | + \frac{  \log( 1/ \delta_{2,m,i} ) }{ n_m }.
\end{align*}
Since $q_m  = n_m /  N $, 
We have
    \begin{align*}
        \bbP \Bigg[ ~ \frac{ B_m(i) }{ n_m   }   \ge  \frac{  2 \sum_{ t\in E_m} | C_t(i) |    }{ N  }  + \frac{  \log( 1 / \delta_{2,m,i} ) }{ n_m }  ~\Bigg] \le 
        \delta_{2,m,i} .
    \end{align*}

Similar arguments show that $ - B_m(i) /n_m  $ satisfies this bound  with probability $\delta_{2,m,i }$. 

(iv)
Let
\begin{align*}
    a_{2,m,i}  = \frac{  \log( 1 / \delta_{2,m,i}  ) }{ n_m  } .
\end{align*}
Altogether, we have
    \begin{align*}
        & \bbP \Bigg[~  \hat{w}_m(i) \ge  w(i)  +   \frac{2 \sum_{ t\in E_m} | C_t(i) | }{ N  } + a_{1,m,i}     +  a_{2,m,i}    ~\bigg]  \le  \exp \bigg[  -\frac{ a_{1,m,i} ^2 \cdot   n_m  }{3} \bigg]  + \exp [  - a_{2,m,i}  \cdot n_m   ]  , \\*
        & \bbP \Bigg[~  \hat{w}_m(i) \le  w(i)  -   \frac{2 \sum_{t\in E_m} | C_t(i) |  }{ N  } - a_{1,m,i}    -  a_{2,m,i}   ~\bigg]  \le   \exp \bigg[  -\frac{ a_{1,m,i} ^2 \cdot  n_m  }{3} \bigg]  + \exp [  - a_{2,m,i}  \cdot  n_m ] .
    \end{align*}
    Note that $ \sum_{ t\in E_m} | c_t(i) | \le  C_m$.
    For $a_{1,m,i}    =  a_{2,m,i}  =  a   \in ( 0,1)$, we have
    \begin{align*}
        & \bbP \Bigg[~  \hat{w}_m(i) \ge  w(i)  +   \frac{2  C_m }{ N  } + 2 a   ~\bigg]  \le  2\exp \bigg[  -\frac{  a ^2 \cdot   n_m  }{3} \bigg]    , \\
        & \bbP \Bigg[~  \hat{w}_m(i) \le  w(i)  -   \frac{2  C_m  }{ N  } - 2  a      ~\bigg]  \le   2 \exp \bigg[  -\frac{  a ^2 \cdot  n_m  }{3} \bigg]   .
    \end{align*}                

\end{proof}

\subsection{Proof of Lemma~\ref{lemma:ub_probaSS_conf_apply_gap} }

\label{pf:lemma_ub_probaSS_conf_apply_gap}

\lemmaUbprobaSSConfApplyGap*

\begin{proof}

First of all,
\begin{align*}
    \{ 1 \in A_{m-1}, 1 \notin A_m, i \in A_m  \}   
    &     \subset
     \{  1,i  \in A_{ m - 1 } , i \in A_m, \hatw_{ m } ( 1  )  \le \hatw_{ m  } ( j ) ~~\forall j \in A_{ m  } \}  \\
    & \subset
    \{  1, i \in A_{ m  - 1 } ,  \hatw_{ m  } ( 1  )  \le \hatw_{ m } ( i )  \} .
\end{align*}
Assume $\calE_{ m, 1   }^{ (\rmL) } (a_i) $ and $\calE_{ m,i   }^{ (\rmU) }  (a_i)$ hold. 
We have
$$ w(1) - \frac{2 C_m }{ N  } - 2 a_i
    < \hat{w}_m(1) \le \hat{w}_m(i)
    < w(i) + \frac{2 C_m }{ N  } + 2 a_i  .$$
In other words,
$$ w(1) - w(i) < \frac{ 4C_m  }{ N  } + 4 a_i .   $$
Note that $ a_i = \Delta_{1,i} / 8$,
we have
\begin{align*}
    & \Delta_{1,i}  < \frac{ 4 C_m }{ N  } + \frac{ \Delta_{1,i} }{2 } 
    ~\Rightarrow~
    \Delta_{1,i}  < \frac{ 8 C_m }{ N } 
\end{align*}
as desired.

\end{proof}

\subsection{Final steps to prove Theorem~\ref{thm:ub_probaSS} }
\label{pf:ub_probaSS_thm_final_step}

 
(i) Assume 
$  \calE_{ m_1,1  }^{ (\rmL) } ( a_{ j_1 } )$ and $  \calE_{ m_1, j_1  }^{  ( \rmU ) } ( {a}_{ j_1 } ) $    
hold. 

\noindent \underline{\bf Case 1}: $1 \neq \iout$.
Lemma~\ref{lemma:ub_probaSS_conf_apply_gap} implies that for any realization of $j_1$, $m_1$,
\begin{align*}
    \Delta_{ 1 , j_{ 1} } \le \frac{ 8 C_{ m_1 } }{    N }  . 
\end{align*}
Since $\iout \in A_m$ for all $1\le m\le M$,  we have $\iout \in A_{ m_1 }$. In addition, since
\begin{align*}
    j_1 := \argmin_{ i \in A_{ m_1}  } w(i),
\end{align*}
we have $\Delta_{ 1, \iout } \le \Delta_{ 1,j_1 }$. Therefore,
\begin{align*}
    \Delta_{ 1, \iout } \le \Delta_{ 1,j_1 } \le \frac{ 8 C_{ m_1 } }{    N }   \le \frac{ 8 C  }{  N }  .
\end{align*} 
 
\noindent \underline{\bf Case 2}: $1 = \iout$. It is trivial to see  $\Delta_{ 1, \iout }  
    \le 8C / N 
    $.

Hence, when $  \calE_{ m_1,1  }^{ (\rmL) } ( a_{ j_1 } )$ and $  \calE_{ m_1, j_1  }^{  ( \rmU ) } ( {a}_{ j_1 } ) $     hold, 
we always have
$\Delta_{ 1, \iout }  
    \le { 8C  } / {N }$.

(ii) Altogether, for any realization of $m_1$, $j_1$,
\begin{align}
    &
    \bbP \Big[   \Big\{ \Delta_{ 1 , \iout } > \frac{ 8C }{   N } \Big\} \bigcap \{ m_1=m, j_1 = i \}    \Big]
    \le
    \bbP \Big[  \Big( ~ \overline{    \calE_{ m_1,1  }^{ (\rmL) } ( a_{ j_1 }  ) \bigcap \calE_{ m_1, j_1  }^{  ( \rmU ) } ( {a}_{ j_1 } )      }~ \Big) \bigcap \{ m_1=m, j_1 = i  \}   \Big].
    \label{eq:ub_probaSS_pf_final_gap_to_event}
\end{align}
In addition, we have
\begin{align}
    &
     \bbP \Big[  \Big( ~ \overline{    \calE_{ m_1,1  }^{ (\rmL) } ( a_{ j_1 }  ) \bigcap \calE_{ m_1, j_1  }^{  ( \rmU ) } ( {a}_{ j_1 } )      }~ \Big) \bigcap \{ m_1=m, j_1 = i  \}   \Big]
    \nonumber \\ &
    \le \bbP \big[ ~ \overline{ \calE_{ m_1,1  }^{ (\rmL) } ( a_{ j_1 }  )   } \bigcap \{ m_1=m, j_1 = i  \}   \big]
    +
    \bbP \big[  ~ \overline{  \calE_{ m_1, j_1  }^{  ( \rmU ) } ( {a}_{ j_1 } )     } \bigcap \{ m_1=m, j_1 = i  \}   \big] 
    \nonumber \\ & 
    \le \bbP \big[ ~ \overline{\calE_{ m_1,1  }^{ (\rmL) } ( a_{ j_1 }  )    } \bigcap \{ u  \cdot i \ge  |A_{m-1}| \} \big] + \bbP \big[~  \overline{  \calE_{ m_1, j_1  }^{  ( \rmU ) } ( {a}_{ j_1 } )    } \bigcap \{ u \cdot i \ge  |A_{m-1}| \}  \big]
    \label{eq:ub_probaSS_pf_final_error_per_phase___bound_size_Am} 
    \\ &
    \le 4 \exp \bigg[   - \frac{   {a}_{ i  }^2 \cdot n_m   }{3}   \bigg] \cdot \mathbb{I}  \{ u \cdot  i \ge  |A_{m-1}| ~ \}
    \label{eq:ub_probaSS_pf_final_error_per_phase___apply_lemma_conc_apply} 
    \\ &
    = 4 \exp \bigg[   - \frac{    \Delta_{1,i}^2 \cdot N   }{  192 |A_{m-1}| }   \bigg] \cdot \mathbb{I}  \{ u \cdot i \ge   |A_{m-1}| ~ \}
    \label{eq:ub_probaSS_pf_final_error_per_phase___apply_def_tilde_a} 
    \\ &
    \le 4 \exp \bigg[   - \frac{   \Delta_{1,i}^2 \cdot N   }{  192  \cdot \min\{ u  \cdot i, L \}  }   \bigg] \label{eq:ub_probaSS_pf_final_error_per_phase0} \\
    & \le 4 \exp \bigg[   - \frac{   N   }{  192 \tilH_2 (w,L,u) }   \bigg] .
    \label{eq:ub_probaSS_pf_final_error_per_phase}
\end{align}
Line \eqref{eq:ub_probaSS_pf_final_error_per_phase___bound_size_Am} results from the  definitions of $j_1$, $A_m~(1\le m \le M)$, which implying that
\begin{align*}
    j_1 \ge | A_{m_1} | =  \bigg\lceil \frac{L}{ u^{m_1} } \bigg\rceil
    =  \bigg\lceil \frac{L}{ u\cdot u^{m_1-1} } \bigg\rceil 
    \ge  \frac{     \big\lceil \frac{L}{   u^{m_1-1} } \big\rceil   }{  u  }
    = \frac{   | A_{m_1 -1 } | }{  u  }.
\end{align*}
Line \eqref{eq:ub_probaSS_pf_final_error_per_phase___apply_lemma_conc_apply} follows from Lemma~\ref{lemma:ub_probaSS_conc_apply}.
Line \eqref{eq:ub_probaSS_pf_final_error_per_phase___apply_def_tilde_a} applies the definitions of $ {a}_{ i  } $ and $n_m$ for all $i,v,m$: 
\begin{align*}
    & {a}_{i }  = \frac{ \Delta_{1,i} }{  8}  
     ~~ \forall \, 2 \le i \le L, \quad \mbox{and}\quad  
    n_m = \frac{ N}{ |A_{m-1}|} ~~ \forall\, 1\le m \le M .
\end{align*}
Lines  \eqref{eq:ub_probaSS_pf_final_error_per_phase0} and \eqref{eq:ub_probaSS_pf_final_error_per_phase} result from the fact that  $|A_m|\le L$ for all $m$ and the definition of $\tilH_2 (w,L,u)$ in \eqref{eq:ub_probaSS_def_HtwoPrime}, i.e., 
\begin{align*}
    \tilH_2 (w,L,u) = \max_{ i \ne 1 } \frac{ \min\{ u \cdot i, L \} }{ \Delta_{1,i}^2 }.
\end{align*}

(iii) Combining~\eqref{eq:ub_probaSS_pf_final_gap_to_event} and \eqref{eq:ub_probaSS_pf_final_error_per_phase}, we have
\begin{align*}
    & \bbP\bigg[  \Delta_{ 1 , \iout } > \frac{ 8 C }{   N }    \bigg]
    \le \sum_{m=1}^M \sum_{i=2}^L 
        \bbP\bigg[  \Big\{ \Delta_{ 1 , \iout } > \frac{ 8C }{   N } \Big\} \bigcap \{ m_1=m, j_1 = i    \} \bigg]
    \le 4M(L-1) \exp \bigg[   - \frac{    N   }{  192 \tilH_2 (w,L,u) }   \bigg].
\end{align*}
We complete the proof of Theorem~\ref{thm:ub_probaSS} with $N =  \lfloor T/M \rfloor$, $M = \lceil \log_u L \rceil$.


\subsection{Proof of Theorem \ref{thm:bai_lb}}\label{pf:thm_bai_lb}
\thmcoupling*
\begin{proof}
We fix $w = \{w(i)\}_{i\in [L]}$, where $1 > w(1) > w(2) >  w(3)\geq \ldots \geq w(L) >0$, and we define $\Delta = w(1) - w(2)$. We assume 
$w(2) -\Delta > w(3) > 0$. We prove the Theorem by a coupling argument between two Bernoulli instances ${\cal I}$, ${\cal I'}$, both on the ground set $[L]$. Both involve $T$ time steps and corruption budget $C = (1 + \lambda)2\Delta T$.

In instance ${\cal I}$, the uncorrupted reward distribution of item $i$ is $\text{Bern}(w(i))$, and the adversary corrupts the rewards of item 1 probabilistically, as detailed in the forthcoming coupling in Algorithm \ref{alg:corrupt_bai}. In instance  ${\cal I'}$ the uncorrupted reward distribution of the items are:
\begin{itemize}[ itemsep = .5pt,   topsep = -1pt , leftmargin = 11pt]
    \item $\text{Bern}(u(1))$, where $u(1) = w(2) - \Delta$, for item 1,  
    \item $\text{Bern}(u(i))$, where $u(i) = w(i)$, for item $i\in [L]\setminus \{1\}$, 
\end{itemize}
but the adversary does not corrupt any of the rewards on instance ${\cal I'}$. The optimal items in instances ${\cal I}, {\cal I'}$ are different, and they are item 1, item 2 respectively. Both instances have optimality gap $\Delta$, since in instance ${\cal I'}$ we have $u(2) > u(1) > u(3) \geq \ldots \geq u(L) > 0$.

We denote the original and corrupted rewards of item $i$ at time step $t$ in instance ${\cal I}$ as $W_t(i), \tilde{W}_t(i)$ respectively, and the original and corrupted rewards of item $i$ at time step $t$ in instance ${\cal I}'$ as $U_t(i), \tilde{U}_t(i)$ respectively. Since there is no corruption on ${\cal I'}$, we have $U_t(i) = \tilde{U}_t(i)$ for all $t, i$ always. 

Fix a BAI algorithm $\pi$, and considering running $\pi$ on the instances ${\cal I}, {\cal I'}$. When $\pi$ is randomized, we assume that $\pi$ has the same random seed in the two runs, so that $\pi$ recommends the same item in both instances ${\cal I}, {\cal I'}$ if $\tilde{W}_t(i) = \tilde{U}_t(i)$ for all $t, i$.  Now, we couple the instances as shown ${\cal I}, {\cal I'}$ in Algorithm \ref{alg:corrupt_bai}.

\begin{algorithm}[ht] 
	\caption{Coupling on instances ${\cal I, \cal I'}$} \label{alg:corrupt_bai}
		\begin{algorithmic}[1]
			\STATE Set remaining corruption budget $B \leftarrow C$.
			\FOR{time step $t = 1, \ldots, T$} 
				\STATE Adversary observes $\{W_t(i)\}_{i\in [L]}$, where $W_t(i)\sim \text{Bern}(w(i))$.
				\STATE Adversary generates $G_t\sim \text{Bern}(2\Delta / w(1))$, independent of $W_t$.
               \IF{$B\geq 1$} \label{alg:bai_coupling_if}
               		\IF{$W_t(1) = 0$}
               		\STATE Set $\tilde{W}_t(1) \leftarrow 0$.
               		\ELSIF{$W_t(1) = 1$, $G_t = 1$}
               \STATE Set $\tilde{W}_t(1) \leftarrow 0$ ($c_t(1) = -1$). 
               \STATE Update $B \leftarrow B - 1$. 

               		\ELSIF{$W_t(1) = 1$, $G_t = 0$}
               		\STATE Set $\tilde{W}_t(1) \leftarrow 1$.
               		\ENDIF
               \STATE Set $\tilde{W}_t(i) \leftarrow W_t(i)$ for all $i\in [L]\setminus \{1\}$.
               \STATE Set $U_t(i)\leftarrow\tilde{W}_t(i)$, $ \tilde{U}_t(i) \leftarrow \tilde{W}_t(i)$ for all $i\in [L]$.
               \ELSE
               \STATE Set $\tilde{W}_t(i) \leftarrow W_t(i)$ for all $i\in [L]$ ($c_t(i) = 0$).
               \STATE Set $U_t(i) \leftarrow \tilde{W}_t(i)$, $\tilde{U}_t(i) \leftarrow \tilde{W}_t(i)$ for all $i\in [L]\setminus \{1\}$.
               \STATE Sample $U_t(1) = \tilde{U}_t(1)\sim \text{Bern}((w(2) - \Delta))$ (recall $u(1) = w(2) - \Delta$). \label{alg:bai_force}
               \ENDIF
			\ENDFOR			
		\end{algorithmic}
	\end{algorithm}

We make two crucial observation on the coupling in Algorithm \ref{alg:corrupt_bai}:
\begin{enumerate}
\item If the corruption budget $C$ is sufficient, that is if we have $B \geq 1$ at the start of time step $T$, then $\tilde{W}_t(i) = \tilde{U}_t(i)$ for all $i, t$, so that the algorithm $\pi$ recommends the same item in both instances.
\item The coupling is valid, in the sense that:
\begin{enumerate}
\item The corruption budget is never exceeded,
\item We always have $W_t(i) \sim \text{Bern}(w(i))$, 
\item We always have $\tilde{U}_t(i) = U_t(i) \sim \text{Bern}(u(i))$. 
\end{enumerate}
The claims (a, b) are clearly true, and for claim (c), we need to verify that $\tilde{U}_t(1) = U_t(1) \sim \text{Bern}(u(1))$. Indeed, at a time step $t$:
\begin{itemize}
\item If $B < 1$, then Line \ref{alg:bai_force} imposes that $U_t(1) \sim \text{Bern}(u(1))$.
\item If $B \geq 1$, then by the \textbf{if} loop in Line \ref{alg:bai_coupling_if}, we have
\begin{align}
&\bbP [ U_t(1) = \tilde{U}_t(1) = 1 | B \geq 1\text{ at the start of time step t} ]\nonumber\\
& =  \bbP [  W_t(1) = 1, G_t = 0) =\bbP [  W_t(1) = 1 ] \cdot \bbP [  G_t = 0 ] \nonumber\\
&  =  w(1) \cdot \left(1 - \frac{2\Delta}{w(1)}\right) = w(1) - 2\Delta = w(2) - \Delta = u(1) \nonumber.
\end{align}
\end{itemize}
\end{enumerate}
The key to the proof is that the optimal item in instances ${\cal I, \cal I'}$ are $1$, $2$ respectively which are different item. By observation 1, if $B\geq 1$ at the start of time $T$, then the algorithm $\pi$ cannot identify the optimal item in both instances. Denote events $\calA_1 = \{\text{$\pi$ outputs item 1 in ${\cal I}$\}}$ and $\calA_2 = \{\text{$\pi$ outputs item 2 in ${\cal I}'$}\}$, and denote $\bbP$ as the probability measure under the coupling in Algorithm \ref{alg:corrupt_bai} and the algorithm $\pi$. Now, 
\begin{align}
&  \bbP[ \calA_1 \cap \calA_2 ] \nonumber\\
 &\leq \bbP[~ \text{$\pi$ outputs different items on ${\cal I, \cal I'}$} ~]  \nonumber\\
& \leq \bbP[ ~ \tilde{W}_t(1) \neq \tilde{U}_t(1) \text{ for some $t\in [T]$}~]\nonumber\\
& \leq \bbP[ ~ \text{At the start of time step $T$, we have $B<1$}~] \nonumber\\
&= \bbP \left[ ~\sum^{T-1}_{t=1}  \mathbb{I}\{ W_t(1) = 1, G_t = 1 \} > C - 1\right] \nonumber\\
&\leq  \bbP \left[ ~\sum^{T}_{t=1}  \mathbb{I}\{ W_t(1) = 1, G_t = 1\}  > (1+ \lambda)2\Delta T\right] \nonumber\label{eq:corrupt_bai_step1}.
\end{align}
To this end, note that the random variables in 
$\{ \mathbb{I}\{ W_t(1) = 1, G_t = 1\}  \}^T_{t=1}$ 
are i.i.d.\ with mean 
$$\mathbb{E} [ \mathbb{I}\{  W_t(1) = 1, G_t = 1 \} ] = \mathbb{E}[ \mathbb{I}\{ W_t(1) = 1\} ] \cdot \mathbb{E}[ \mathbb{I}\{  G_t = 1\} ] = w(1) \cdot \frac{ 2\Delta }{w(1) }  = 2\Delta.$$  
By applying Theorem \ref{thm:conc_bernoulli}, we have 
$$
\bbP[ \calA_1 \cap \calA_2 ]  \leq \bbP \left[ \sum^{T}_{t=1} \mathbb{I}\{ W_t(1) = 1, G_t = 1)\} > (1+ \lambda)2\Delta T\right] \leq \exp\left(-\frac{2\lambda^2 \Delta T }{3}\right).
$$
Finally, we have
$$
\bbP[ \calA_1 ] + \bbP[ \calA_2 ] = \bbP[ \calA_1 \cup \calA_2] + \bbP[ \calA_1 \cap \calA_2 ] \leq 1 + \exp\left(-\frac{2\lambda^2 \Delta T }{3}\right),
$$
so that
$$
\min\left\{ \bbP[ \calA_1] ,\bbP[ \calA_2 ] \right\} \leq \frac{1}{2} \left[1 + \exp\left(-\frac{2\lambda^2 \Delta T }{3}\right)\right],
$$
completing the proof of the theorem. 
\end{proof}

\subsection{Proof of Theorem \ref{thm:bai_lb_sh}  }

\thmcouplingSH*

\begin{proof}
Consider Theorem~\ref{thm:bai_lb}'s attack strategy, but applied to SH in phase~1. 

We claim that there is a BAI instance ${\cal I}$ with $T$ time steps, gap $\Delta  $ and $C = (1+\lambda)2\Delta T / (L\log_2 L) $, such that $$\Pr(\Delta_{1, i_\text{out}} \geq \Delta )  = \Pr \left(\Delta_{1, i_\text{out}} \geq \frac{C \cdot L\log_2 L}{2(1+\lambda)T)} \right)\geq 1/2$$ when $T$ is sufficiently large. This is a matching lower bound for SH in Table 1. 

Consider a Bernoulli instance $\{w(i)\}_{i\in [L]}$ with $w(1)\in [1/2, 1]$ and $w(i) = w(1) - \Delta$ for $i\in [L]\setminus \{1\}$. In phase 1, SH pulls each $i\in [L]$ for $\tau = \lceil T / (L \log_2 L)\rceil$ times, computes the empirical means $\{\hat{w}_1(i)\}_{i\in [L]}$, and removes the $\lceil L / 2\rceil$ items with smallest $\hat{w}_1(i)$ from consideration. 

During phase $1$, SH pulls item $1$ at fixed time steps $\{t_s\}^\tau_{s = 1}$. When the adversary determines $\{c_{t_s}(i)\}_{i\in [L]}$, \emph{he knows $\{W_{t_s}(i)\}_{i\in [L]}$, and knows that item $1$ will be pulled at  time  $t_s$.} The adversary attacks by solely corrupting item $1$ solely at times $\{t_s\}^\tau_{t=1}$. 

If the corruption budget is not exhausted, set 
\begin{align*}
    & \Pr(\tilde{W}_{t_s}(1) = 0 | W_{t_s}(1) = 1) = \frac{2\Delta }{w(1) }= 1 - \Pr(\tilde{W}_{t_s}(1) = 1 | W_{t_s}(1) = 1),  \\
    & \Pr(\tilde{W}_{t_s}(1) = 0 | W_{t_s}(1) = 0) = 1,
\end{align*}
which implies that
\begin{align*}
    & \Pr(\tilde{W}_{t_s}(1) = 1)  =w(1) - 2\Delta, \\
    & \Pr(c_{t_s}(1) = -1) = \Pr(\tilde{W}_{t_s}(1) = 0 | W_{t_s}(1)=1) \Pr(W_{t_s}(1) = 1) = 2\Delta.
\end{align*}
%
%
If exhausted, no corruption. 

Let $X_1, \ldots, X_\tau \sim\text{Bern}(w(1) - 2\Delta) $, $Y_1, \ldots, Y_\tau \sim\text{Bern}(2\Delta)$ be i.i.d.\ random variables and event 
$${\cal E}:=\{\text{corruption budget not exhausted at end of phase 1}\}.$$
Then 
\begin{align*}
      \Pr\left( \hat{w}_1(i)   \leq  w(1)  -  \frac{3\Delta}{2}\right)  
    & \geq  \Pr\left( \hat{w}_1(i)  \leq  w(1)   -   \frac{3\Delta}{ 2 }  \bigg| {\cal E}\right)  \cdot \Pr\left( {\cal E} \right) \\
    &  =  \Pr\left(\frac{1}{\tau}\sum^\tau_{s = 1} X_s   \leq  w(1) -   \frac{3\Delta}{2} \right) \cdot \Pr\left(\sum^\tau_{s= 1} Y_s  \leq  \frac{(1   + \lambda)2\Delta T }{L \log_2 L}\right) ,
\end{align*}
which exceeds $1 - \exp\left[-\frac{T\Delta^2}{2 L \log_2 L}\right] - \exp\left[- \frac{2\lambda^2 T \Delta} {3L\log_2 L}\right]$. 
Thus, for all $i\ne 1$, 
$$\Pr\left(\hat{w}_1(i) > w(1) - \frac{3\Delta}{2} \right) \geq 1 - \exp \left[-\frac{T\Delta^2}{2 L \log_2 L}\right] .$$ 
Lastly, by the union bound, 
$$\Pr(\text{Item $1$ removed after phase 1}) \geq 1 - (L+1)\exp \bigg[-\frac{T\max\{1, \lambda^2\}\Delta^2}{2 L \log_2 L} \bigg].$$
We complete the proof by noting that
\begin{align*}
  1 - (L+1)\exp \bigg[-\frac{T\max\{1, \lambda^2\}\Delta^2}{2 L \log_2 L}\bigg] \ge 1/2
\end{align*}
for $T$ large enough.


\end{proof}

\subsection{Proof of Theorem~\ref{thm:lb_ber_shift_to_zero_one} }

\label{pf:thm_lb_ber_shift_to_zero_one}

\thmLbBerShiftToZeroOne*

\begin{proof}

\noindent
{\bf Part (a).}
Let $X_{1}, \cdots, X_{n}$ be Bernoulli random variables taking values in $\{0,1\}$ such that $\mathbb{E}\left[X_{t} | X_{1}, \cdots, X_{t-1}\right] \leq \mu$ for all $t \leq n,$ and $Y=X_{1}+\ldots+X_{n}$. 
Let $X'_t = 1-X_t$ for all $1\le t \le n$, $\mu' = 1-\mu$, $Y' = X'_1 + \ldots + X'_n$.
Then, Theorem~\ref{thm:conc_bernoulli} indicates for all $\delta \in(0,1)$
\begin{align*}
    & \Pr [Y' \leq(1-\delta) n \mu'] \leq \exp \bigg( -\frac{\delta^{2} n \mu'}{2} \bigg)
    ~\Rightarrow~
        \Pr [ n - Y \leq(1-\delta) n ( 1 - \mu) ] \leq \exp \bigg[ -\frac{\delta^{2} n (1-\mu) }{2} \bigg]
    \\ \Rightarrow~ &
    \Pr [   Y \geq n- (1-\delta) n ( 1 - \mu) ] \leq \exp \bigg[ -\frac{\delta^{2} n (1-\mu) }{2} \bigg].
\end{align*}

(i) For all $i\in [L]$,  $W_t(i)$ denotes the random reward of item $i$ at time step $t$.
Fix any $\lambda \in (0,1)$.  We can apply the inequality above with $\mu = w(1)$, $n=TL$ to get
\begin{align*}
    &
    \bbP \Bigg[  \sum_{ t =1 }^T \sum_{i=1}^L W_t(i )  \ge TL- (1 - \lambda ) TL [ 1-w(1) ]  \Bigg] \le \exp\bigg[ - \frac{ \lambda^2 TL [1- w(1) ] }{2}  \bigg] 
    \\ \Rightarrow~ &
    \bbP \Bigg[    \sum_{ t =1 }^T \sum_{i=1}^L   \mathbb{I} \{ W_t(i ) = 1 \}  \ge TL \cdot \{ 1 - (1 - \lambda )[ 1-w(1) ] \} \Bigg] \le \exp\bigg[ - \frac{ \lambda^2 TL [1- w(1) ] }{2}  \bigg] .
\end{align*}
(ii) Let
\begin{align*}
    &
    \calE_{ \lambda, 0}:= 
    \Bigg\{ \sum_{ t =1 }^T  \sum_{i=1}^L    \mathbb{I} \{ W_t(i ) = 1 \}  <
          T L \cdot \{ 1 - (1 - \lambda )[ 1-w(1) ] \}
          \Bigg\}.
\end{align*}

When $\calE_{ \lambda, 0}$ holds, throughout the whole horizon ($T$ time steps), there are less than $ TL \cdot \{ 1 - (1 -  \lambda )[ 1-w(1) ] \}$ random rewards that equal to $1$. 
If we additionally have
\begin{align*}
    C \ge T L \cdot \{ 1 - (1 -  \lambda )[ 1-w(1) ] \}  :=C_{ \lambda  , 0},
\end{align*}
the adversary can shift the random reward to $0$ whenever it equals to $1$, which implies that the agent get a corrupted reward equals to $0$ at each time step.

Altogether, when $\calE_{  \lambda ,0 } $ holds and $C\ge C_{ \lambda ,0 } $, the agent get a corrupted reward equals to $0$ at each time step. 
Therefore, the observations of random rewards throughout the whole horizon provides no information about the mean reward $w(i)$ for any item $i\in [L]$.
In this case, the best method for the agent to output an item is to randomly output any ground item with a uniform probability of $1/L$.
As a result, for any item $i$,
\begin{align*}
    \bbP\big[ \{  \iout = i  \}  \bigcap \calE_{ \lambda,0 } \big] \le \frac{1}{L}.
\end{align*}
Recall that $
    L_{ \epsilon  } : = | \{ i\in [L] : \Delta_{1,i } \le \epsilon \} | 
$
 counts the items with mean reward at most $\epsilon$ worse than that of the optimal item, we have 
\begin{align*}
    \bbP\big[ \{ \Delta_{1,\iout} \le \epsilon\}  \bigcap \calE_{ \lambda, 0 } \big] 
    \le \sum_{ i\in [L],  \Delta_{1, i } \le \epsilon } \bbP\big[ \{  \iout = i  \}  \bigcap \calE_{ \lambda, 0 } \big]
    \le \frac{ L_{ \epsilon  } }{L}.
\end{align*}

(iii) Therefore,
\begin{align*}
    &
    \bbP [ \Delta_{1,\iout} \le \epsilon ] 
    = \bbP\big[ \{  \Delta_{1,\iout} \le \epsilon  \}  \bigcap \calE_{ \lambda,0 } \big] + \bbP\big[ \{  \Delta_{1,\iout} \le \epsilon  \}  \bigcap \overline{ \calE_{ \lambda,0 } } ~ \big]
    \\ &
    \le  \bbP\big[ \{  \Delta_{1,\iout} \le \epsilon  \}  \bigcap \calE_{ \lambda,0 }  \big] + \bbP\big[ ~ \overline{ \calE_{ \lambda,0 } } ~ \big]
    \le \frac{L_\epsilon }{L} + \exp\bigg[ - \frac{ \lambda^2 T L [ 1- w(1) ]  }{2}  \bigg],
\end{align*}
Lastly,
\begin{align*}
    \bbP [ \Delta_{1,\iout} > \epsilon ]  
    \ge 1 - \frac{L_\epsilon }{L} - \exp\bigg[ - \frac{ \lambda^2 T L [ 1- w(1) ]   }{2}  \bigg].
\end{align*}


%
%
%
%
%

\noindent {\bf Part (b).}
(i) For all $i\in [L]$,  $W_t(i)$ denotes the random reward of item $i$ at time step $t$.
Fix any $\lambda \in (0,1)$. We can apply Theorem~\ref{thm:conc_bernoulli} with $\mu = w(L)$, $n=TL$ to get
\begin{align*}
    \bbP \Bigg[  \sum_{ t =1 }^T  \sum_{i=1}^L  W_t(i )  \le (1 - \lambda ) TL w(L)  \Bigg] \le \exp\bigg[ - \frac{ \lambda ^2 T L w(L) }{2}  \bigg].
\end{align*}
Meanwhile,
\begin{align*}
    &
    \Bigg\{ \sum_{ t =1 }^T \sum_{i=1}^L   W_t( i )  \le (1 - \lambda ) T L w(L) \Bigg\}
    = \Bigg\{ \sum_{ t =1 }^T  \sum_{i=1}^L   \mathbb{I} \{ W_t(i ) = 1 \}  \le (1 - \lambda )  TL w(L)  \Bigg\}
    \\ &
    = \Bigg\{  \sum_{ t =1 }^T  \sum_{i=1}^L   \mathbb{I} \{ W_t(i ) = 0 \}  \ge T L - (1 - \lambda ) T L w(L) 
        = T L \cdot [ 1 - (1- \lambda ) w(L) ] \Bigg\}.
\end{align*}
Therefore,
\begin{align*}
    \bbP \Bigg[  \sum_{ t =1 }^T   \mathbb{I} \{ W_t(i ) = 0 \}  \ge  
          T L \cdot [ 1 - (1-\lambda ) w(L) ] 
         \Bigg] \le \exp\bigg[ - \frac{ \lambda ^2 T L  w(L) }{2}  \bigg].
\end{align*}

(ii) Let
\begin{align*}
    &
    \calE_{ \lambda ,1 } := 
    \Bigg\{ \sum_{ t =1 }^T   \sum_{i=1}^L  \mathbb{I} \{ W_t( i ) = 0 \}  <
          T L \cdot [ 1 - (1-\lambda ) w(L) ] 
          \Bigg\}.
\end{align*}

When $\calE_{ \lambda ,1 }$ holds, throughout the whole horizon ($T$ time steps), there are less than $ T  L \cdot [ 1 - (1- \lambda ) w(L) ] $ random rewards that equal to $0$. 
If we additionally have
\begin{align*}
    C \ge T L  \cdot [ 1 - (1- \lambda ) w(L) ]  :=C_{ \lambda ,1 } ,
\end{align*}
the adversary can shift the random reward to $1$ whenever it equals to $0$, which implies that the agent get a corrupted reward equals to $1$ at each time step.

Altogether, when $\calE_{ \lambda ,1 }$ holds and $C\ge C_{ \lambda ,1 }$, the agent get a corrupted reward equals to $1$ at each time step. 
Therefore, the observations of random rewards throughout the whole horizon provides no information about the mean reward $w(i)$ for any item $i\in [L]$.
In this case, the best method for the agent to output an item is to randomly output any ground item with a uniform probability of $1/L$.
As a result, for any item $i$,
\begin{align*}
    \bbP\big[ \{  \iout = i  \}  \bigcap \calE_{ \lambda,1 } \big] \le \frac{1}{L}.
\end{align*}
Recall that $
    L_{ \epsilon  } : = | \{ i\in [L] : \Delta_{1,i } \le \epsilon \} | 
$
 counts the items with mean reward at most $\epsilon$ worse than that of the optimal item, we have 
\begin{align*}
    \bbP\big[ \{ \Delta_{1,\iout} \le \epsilon\}  \bigcap \calE_{ \lambda,1 } \big] 
    \le \sum_{ i\in [L],  \Delta_{1, i } \le \epsilon } \bbP\big[ \{  \iout = i  \}  \bigcap \calE_{ \lambda,1 } \big]
    \le \frac{ L_{ \epsilon  } }{L}.
\end{align*}

(iii) Therefore,
\begin{align*}
    &
    \bbP [ \Delta_{1,\iout} \le \epsilon ] 
    = \bbP\big[ \{  \Delta_{1,\iout} \le \epsilon  \}  \bigcap \calE_{ \lambda,1 } \big] + \bbP\big[ \{  \Delta_{1,\iout} \le \epsilon  \}  \bigcap \overline{ \calE_{ \lambda,1 } } ~ \big]
    \\ &
    \le  \bbP\big[ \{  \Delta_{1,\iout} \le \epsilon  \}  \bigcap \calE_{ \lambda,1 }  \big] + \bbP\big[ ~ \overline{ \calE_{ \lambda,1 } } ~ \big]
    \le \frac{L_\epsilon }{L} + \exp\bigg[ - \frac{ \lambda^2 T L w(L) }{2}  \bigg],
\end{align*}
Lastly,
\begin{align*}
    \bbP [ \Delta_{1,\iout} > \epsilon ]  
    \ge 1 - \frac{L_\epsilon }{L} - \exp\bigg[ - \frac{ \lambda^2 T L w(L) }{2}  \bigg].
\end{align*}

\end{proof}


\section{Additional numerical results}
\label{appdix_exp}

\subsection{Details of Figures~\ref{pic:exp_compare_w_lambda_a} and \ref{pic:exp_compare_w_lambda_b} }

\begin{table}[ht] 
  \centering
  \caption{ Comparison of PSS$(u)$ to Other Algorithms
  }
  \vspace{.5em}
	\begin{tabular}{ ccc | ccc }
        $\lambda$    &  $w^*$ & $w'$ & PSS($2$) & SH & UP  \\
	 \thickhline 
	 \vphantom{$\hat{\big[} $ }
    $0.5 $ & $0.4 $ & $0.2 $ & $76  $ & $42  $ & $12 $ \\
    $0.5 $ & $0.5 $ & $0.2 $ & $91  $ & $64  $ & $13 $ \\
    $0.5 $ & $0.5 $ & $0.3 $ & $74  $ & $51  $ & $19 $ \\
    \hline 
    \vphantom{$\hat{\big[}$ }
    $0.9 $ & $0.4 $ & $0.2 $ & $72  $ & $45  $ & $9 $ \\
    $0.9 $ & $0.5 $ & $0.2 $ & $83  $ & $64  $ & $7 $ \\
    $0.9 $ & $0.5 $ & $0.3 $ & $60  $ & $40  $ & $12 $
	    \end{tabular}%
	  \label{tab:exp_compare_w_lambda}%
\end{table}
Here, we provide the raw numbers of for Figures~\ref{pic:exp_compare_w_lambda_a} and \ref{pic:exp_compare_w_lambda_b}. We see that PSS($2$) consistently and clearly outperform the non-robust BAI algorithms on all instances here. 

\begin{figure}[ht]
    \begin{flushright}
    \includegraphics[width = .28\textwidth ]{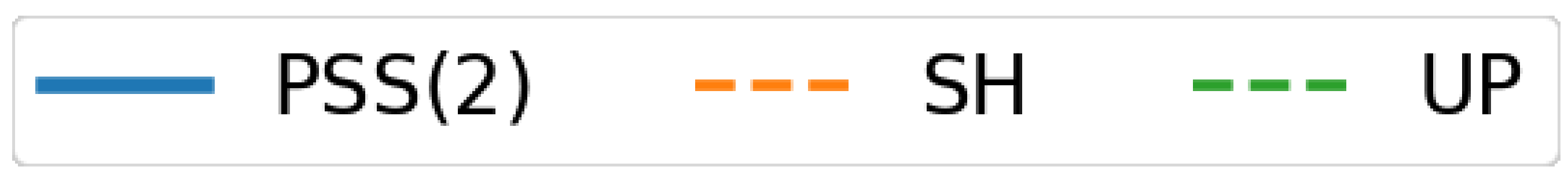} 
    \hphantom{aa}
    %
      \end{flushright}
   \vspace{-1.3em}
  \centering 
  
  \subfigure[$L=32, \lambda=9$]{
\label{pic_supp:exp_supp_compare_T_L32}
\includegraphics[width = \textwidth]{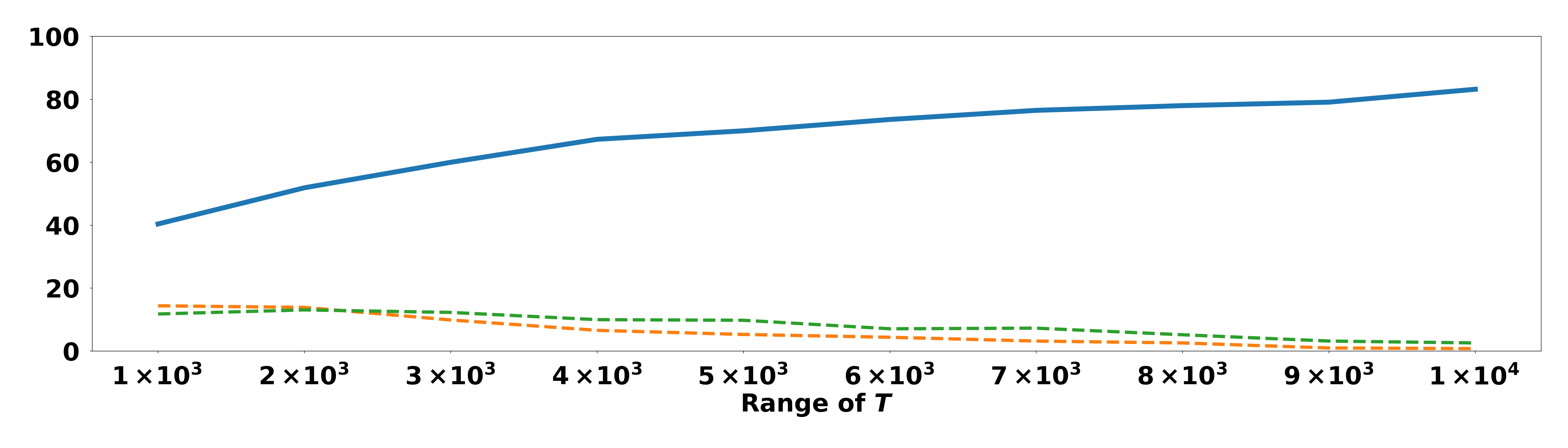}
}\vspace{-.05in}
 
\subfigure[{}$L=64, \lambda=19$]{
\label{pic_supp:exp_supp_compare_T_L64}
\includegraphics[width = \textwidth]{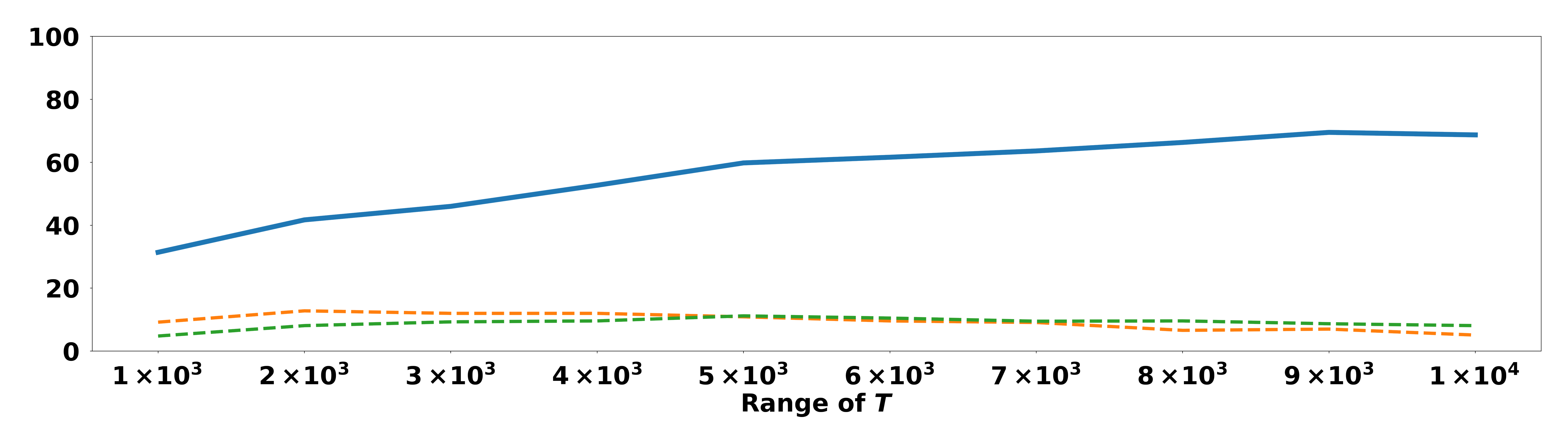}
}\vspace{-.05in}
  
  \subfigure[$L=128, \lambda=39$]{
\label{pic_supp:exp_supp_compare_T_L128}
\includegraphics[width = \textwidth]{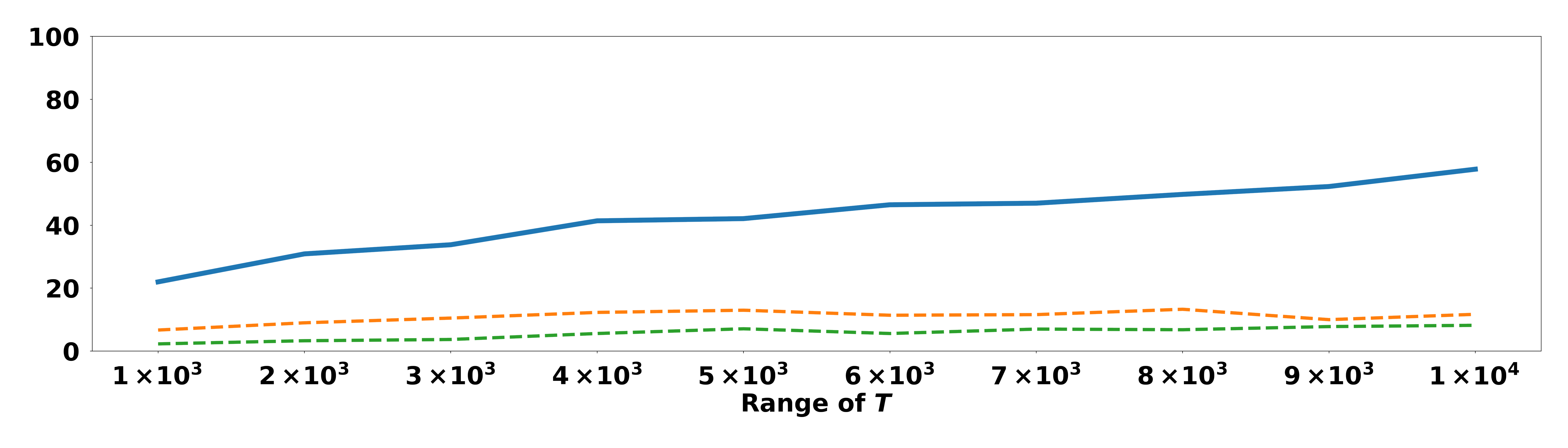}
}\vspace{-.05in}

  \caption{Percentage of correct BAI of PSS($2$), SH and UP. We fix the instance to be  $w^*=0.4, w'=0.2$.   }
  \label{pic_supp:exp_compare}
\end{figure}

\subsection{Further observations}

To further evaluate the impact of $T$, $L$, and $\lambda$ on the success probabilities of PSS$(2)$, SH and UP, we run each algorithm for $1000$ times independently with varying sets of parameters, while keeping the MAB instance at $w^*=0.4$ and $w'=0.2$ fixed. 


Recall that according to Theorem~\ref{thm:bai_lb_sh}, we set the CPS 
\begin{equation}
\frac{C}{T} = \frac{2\Delta\cdot (1+\lambda)  }{  L \log_2 L}. \label{eqn:cps}
\end{equation}
This is the scaling of the CPS that ensures that SH fails with high probability as $T$ grows. Notice that $C/T$ grows with $\lambda$ and decreases with $L$. We implement the attack strategy as applied in Theorem~\ref{thm:bai_lb_sh}  (see Algorithm~\ref{alg:corrupt_bai}) and vary $L$ and $T$. 

In each subplot in Figure \ref{pic_supp:exp_compare}, we consider different number of items $L$ and  use different values of $\lambda$, resulting in different CPSes. We let $\lambda$  grows with $L$, 
so the identification of the best arm would pose significant difficulty to SH as prescribed by Eqn.~\eqref{eqn:cps}. The figures show that as $T$ grows, the success (BAI) probabilities of PSS($2$) demonstrate an increasing trend, and in the case of $L = 32, \lambda = 9$ the percentage of successful BAI converges to 100\%. 
In stark contrast, 
the percentages of successful BAI for SH are always below 20\%. In the case of $L = 32, \lambda = 9$, the percentage appears to converge to 0 as $T$ increases.
This implies that SH fails with high probability when $T$ is sufficiently large, which corroborates Theorem~\ref{thm:bai_lb_sh}. However, the randomization inherent in PSS($2$) ensures that it remains extremely robust to the corruption strategy and it successfully identifies the best item a large fraction of times as $T\to+\infty$; this corroborates our main result---Theorem~\ref{thm:ub_probaSS}.

\end{document}

%% file: preamble_icml_arxiv.tex
\usepackage[mathscr]{eucal}
\usepackage[cmex10]{amsmath}
\usepackage{epsfig,epsf,psfrag}
\usepackage{amssymb,amsmath,amsthm,amsfonts,latexsym}
\usepackage{amsmath,graphicx,bm,xcolor,overpic}
\usepackage{fixltx2e}
\usepackage{array}
\usepackage{verbatim}
\usepackage{bm,bbm}
\usepackage{verbatim}
\usepackage{textcomp}
\usepackage{mathrsfs}
\usepackage{epstopdf}

\usepackage{url}

\makeatother

\usepackage{thmtools}
\usepackage{thm-restate}

\numberwithin{equation}{section}

\numberwithin{table}{section}
\numberwithin{figure}{section}
 
\catcode`~=11 \def\UrlSpecials{\do\~{\kern -.15em\lower .7ex\hbox{~}\kern .04em}} \catcode`~=13 

\usepackage{longtable}
\usepackage{setspace}

\newcommand{\calA}{\mathcal{A}}

\newcommand{\calE}{\mathcal{E}}
\newcommand{\calF}{\mathcal{F}}

\newcommand{\bW}{\mathbf{W}}

\newcommand{\rmL}{\mathrm{L}}

\newcommand{\rmU}{\mathrm{U}}

\newcommand{\bbE}{\mathbb{E}}

\newcommand{\bbN}{\mathbb{N}}

\newcommand{\bbP}{\mathbb{P}}

\DeclareMathAlphabet{\mathbsf}{OT1}{cmss}{bx}{n}
\DeclareMathAlphabet{\mathssf}{OT1}{cmss}{m}{sl}

\DeclareSymbolFont{bsfletters}{OT1}{cmss}{bx}{n}  
\DeclareSymbolFont{ssfletters}{OT1}{cmss}{m}{n}
\DeclareMathSymbol{\bsfGamma}{0}{bsfletters}{'000}
\DeclareMathSymbol{\ssfGamma}{0}{ssfletters}{'000}
\DeclareMathSymbol{\bsfDelta}{0}{bsfletters}{'001}
\DeclareMathSymbol{\ssfDelta}{0}{ssfletters}{'001}
\DeclareMathSymbol{\bsfTheta}{0}{bsfletters}{'002}
\DeclareMathSymbol{\ssfTheta}{0}{ssfletters}{'002}
\DeclareMathSymbol{\bsfLambda}{0}{bsfletters}{'003}
\DeclareMathSymbol{\ssfLambda}{0}{ssfletters}{'003}
\DeclareMathSymbol{\bsfXi}{0}{bsfletters}{'004}
\DeclareMathSymbol{\ssfXi}{0}{ssfletters}{'004}
\DeclareMathSymbol{\bsfPi}{0}{bsfletters}{'005}
\DeclareMathSymbol{\ssfPi}{0}{ssfletters}{'005}
\DeclareMathSymbol{\bsfSigma}{0}{bsfletters}{'006}
\DeclareMathSymbol{\ssfSigma}{0}{ssfletters}{'006}
\DeclareMathSymbol{\bsfUpsilon}{0}{bsfletters}{'007}
\DeclareMathSymbol{\ssfUpsilon}{0}{ssfletters}{'007}
\DeclareMathSymbol{\bsfPhi}{0}{bsfletters}{'010}
\DeclareMathSymbol{\ssfPhi}{0}{ssfletters}{'010}
\DeclareMathSymbol{\bsfPsi}{0}{bsfletters}{'011}
\DeclareMathSymbol{\ssfPsi}{0}{ssfletters}{'011}
\DeclareMathSymbol{\bsfOmega}{0}{bsfletters}{'012}
\DeclareMathSymbol{\ssfOmega}{0}{ssfletters}{'012}

\newcommand{\tilH}{\tilde{H}}

\newcommand{\hatw}{\hat{w}}

\DeclareMathOperator*{\argmin}{arg\,min}

\DeclareMathOperator{\var}{\mathrm{Var}}

\newtheorem{theorem}{Theorem}[section]

%% file: corruption_bai_v26_for_arxiv_upload.bbl
\begin{thebibliography}{26}
\providecommand{\natexlab}[1]{#1}
\providecommand{\url}[1]{\texttt{#1}}
\expandafter\ifx\csname urlstyle\endcsname\relax
  \providecommand{\doi}[1]{doi: #1}\else
  \providecommand{\doi}{doi: \begingroup \urlstyle{rm}\Url}\fi

\bibitem[Altschuler et~al.(2019)Altschuler, Brunel, and
  Malek]{altschuler2019best}
Altschuler, J., Brunel, V.-E., and Malek, A.
\newblock Best arm identification for contaminated bandits.
\newblock \emph{Journal of Machine Learning Research}, 20\penalty0
  (91):\penalty0 1--39, 2019.

\bibitem[Audibert \& Bubeck(2010)Audibert and Bubeck]{audibert2010best}
Audibert, J.-Y. and Bubeck, S.
\newblock Best arm identification in multi-armed bandits.
\newblock In \emph{Proceedings of the 23th Conference on Learning Theory}, pp.\
   41--53, 2010.

\bibitem[Auer et~al.(2002)Auer, Cesa-Bianchi, and Fischer]{auer2002finite}
Auer, P., Cesa-Bianchi, N., and Fischer, P.
\newblock Finite-time analysis of the multiarmed bandit problem.
\newblock \emph{Machine Learning}, 47\penalty0 (2-3):\penalty0 235--256, 2002.

\bibitem[Beygelzimer et~al.(2011)Beygelzimer, Langford, Li, Reyzin, and
  Schapire]{beygelzimer2011contextual}
Beygelzimer, A., Langford, J., Li, L., Reyzin, L., and Schapire, R.
\newblock Contextual bandit algorithms with supervised learning guarantees.
\newblock In \emph{Proceedings of the 14th International Conference on
  Artificial Intelligence and Statistics}, pp.\  19--26, 2011.

\bibitem[Bogunovic et~al.(2020)Bogunovic, Krause, and Scarlett]{Bogunovic0S20}
Bogunovic, I., Krause, A., and Scarlett, J.
\newblock Corruption-tolerant {G}aussian process bandit optimization.
\newblock In \emph{Proceedings of the 23rd International Conference on
  Artificial Intelligence and Statistics}, pp.\  1071--1081, 2020.

\bibitem[Carpentier \& Locatelli(2016)Carpentier and
  Locatelli]{carpentier2016tight}
Carpentier, A. and Locatelli, A.
\newblock Tight (lower) bounds for the fixed budget best arm identification
  bandit problem.
\newblock In \emph{Proceedings of the 29th Conference on Learning Theory}, pp.\
   590--604, 2016.

\bibitem[Chen et~al.(2014)Chen, Lin, King, Lyu, and Chen]{NIPS2014_5433}
Chen, S., Lin, T., King, I., Lyu, M.~R., and Chen, W.
\newblock Combinatorial pure exploration of multi-armed bandits.
\newblock In \emph{Proceedings of the 27th Advances in Neural Information
  Processing Systems}, pp.\  379--387. 2014.

\bibitem[Chen et~al.(2016)Chen, Wang, Yuan, and Wang]{chen2016combinatorial}
Chen, W., Wang, Y., Yuan, Y., and Wang, Q.
\newblock Combinatorial multi-armed bandit and its extension to
  probabilistically triggered arms.
\newblock \emph{Journal of Machine Learning Research}, 17\penalty0
  (1):\penalty0 1746--1778, 2016.

\bibitem[Dubhashi \& Panconesi(2009)Dubhashi and
  Panconesi]{dubhashi2009concentration}
Dubhashi, D.~P. and Panconesi, A.
\newblock \emph{Concentration of measure for the analysis of randomized
  algorithms}.
\newblock Cambridge University Press, 2009.

\bibitem[Gupta et~al.(2019)Gupta, Koren, and Talwar]{gupta2019better}
Gupta, A., Koren, T., and Talwar, K.
\newblock Better algorithms for stochastic bandits with adversarial
  corruptions.
\newblock In \emph{Proceedings of the 32nd Conference on Learning Theory}, pp.\
   1562--1578, 2019.

\bibitem[Jun et~al.(2016)Jun, Jamieson, Nowak, and Zhu]{jun2016top}
Jun, K.-S., Jamieson, K.~G., Nowak, R.~D., and Zhu, X.
\newblock Top arm identification in multi-armed bandits with batch arm pulls.
\newblock In \emph{Proceedings of the 19th International Conference on
  Artificial Intelligence and Statistics}, pp.\  139--148, 2016.

\bibitem[Jun et~al.(2018)Jun, Li, Ma, and Zhu]{jun2018adversarial}
Jun, K.-S., Li, L., Ma, Y., and Zhu, J.
\newblock Adversarial attacks on stochastic bandits.
\newblock In \emph{Proceedings of the 31st Advances in Neural Information
  Processing Systems}, pp.\  3640--3649, 2018.

\bibitem[Karnin et~al.(2013)Karnin, Koren, and Somekh]{karnin2013almost}
Karnin, Z., Koren, T., and Somekh, O.
\newblock Almost optimal exploration in multi-armed bandits.
\newblock In \emph{Proceedings of the 13th International Conference on Machine
  Learning}, pp.\  1238--1246, 2013.

\bibitem[Krishnamurthy et~al.(2020)Krishnamurthy, Lykouris, and
  Podimata]{krishnamurthy2020binsearch}
Krishnamurthy, A., Lykouris, T., and Podimata, C.
\newblock Corrupted multidimensional binary search: Learning in the presence of
  irrational agents.
\newblock \emph{arXiv preprint arXiv:2002.11650}, 2020.

\bibitem[Lattimore \& Szepesv{\'a}ri(2020)Lattimore and
  Szepesv{\'a}ri]{lattimore2020bandit}
Lattimore, T. and Szepesv{\'a}ri, C.
\newblock \emph{Bandit algorithms}.
\newblock Cambridge University Press, 2020.

\bibitem[Li et~al.(2019)Li, Lou, and Shan]{li2019stochastic}
Li, Y., Lou, E.~Y., and Shan, L.
\newblock Stochastic linear optimization with adversarial corruption.
\newblock \emph{arXiv preprint arXiv:1909.02109}, 2019.

\bibitem[Liu \& Shroff(2019)Liu and Shroff]{liu2019data}
Liu, F. and Shroff, N.
\newblock Data poisoning attacks on stochastic bandits.
\newblock In \emph{Proceedings of the 36th International Conference on Machine
  Learning}, pp.\  4042--4050, 2019.

\bibitem[Liu \& Lai(2020)Liu and Lai]{liu2020action}
Liu, G. and Lai, L.
\newblock Action-manipulation attacks on stochastic bandits.
\newblock In \emph{Proceedings of the 45th International Conference on
  Acoustics, Speech and Signal Processing}, pp.\  3112--3116, 2020.

\bibitem[Lykouris et~al.(2018)Lykouris, Mirrokni, and
  Leme]{lykouris2018stochastic}
Lykouris, T., Mirrokni, V., and Leme, R.~P.
\newblock Stochastic bandits robust to adversarial corruptions.
\newblock In \emph{STOC 2018: Proceedings of the 50th Annual ACM SIGACT
  Symposium on Theory of Computing}, pp.\  114--122, 2018.

\bibitem[Lykouris et~al.(2020)Lykouris, Simchowitz, Slivkins, and
  Sun]{lykouris2020corruption}
Lykouris, T., Simchowitz, M., Slivkins, A., and Sun, W.
\newblock Corruption robust exploration in episodic reinforcement learning.
\newblock \emph{arXiv preprint arXiv:1911.08689}, 2020.

\bibitem[Mitzenmacher \& Upfal(2017)Mitzenmacher and
  Upfal]{mitzenmacher2017probability}
Mitzenmacher, M. and Upfal, E.
\newblock \emph{Probability and computing: Randomization and probabilistic
  techniques in algorithms and data analysis}.
\newblock Cambridge University Press, 2017.

\bibitem[Rejwan \& Mansour(2020)Rejwan and Mansour]{rejwan2020top}
Rejwan, I. and Mansour, Y.
\newblock Top-$ k $ combinatorial bandits with full-bandit feedback.
\newblock In \emph{Proceedings of the 31st International Conference on
  Algorithmic Learning Theory}, pp.\  752--776, 2020.

\bibitem[Shen(2019)]{shen2019universal}
Shen, C.
\newblock Universal best arm identification.
\newblock \emph{IEEE Transactions on Signal Processing}, 67\penalty0
  (17):\penalty0 4464--4478, 2019.

\bibitem[Zhong et~al.(2020)Zhong, Cheung, and Tan]{zhong2020best}
Zhong, Z., Cheung, W.~C., and Tan, V. Y.~F.
\newblock Best arm identification for cascading bandits in the fixed confidence
  setting.
\newblock In \emph{Proceedings of the 37th International Conference on Machine
  Learning}, 2020.

\bibitem[Zimmert \& Seldin(2019)Zimmert and Seldin]{zimmert2018tsallisinf}
Zimmert, J. and Seldin, Y.
\newblock An optimal algorithm for stochastic and adversarial bandits.
\newblock In \emph{Proceedings of the 22nd International Conference on
  Artificial Intelligence and Statistics}, pp.\  467--475, 2019.

\bibitem[Zuo(2020)]{zuo2020near}
Zuo, S.
\newblock Near optimal adversarial attack on {UCB} bandits.
\newblock \emph{arXiv preprint arXiv:2008.09312}, 2020.

\end{thebibliography}
